\documentclass{article}

 \usepackage[preprint]{neurips_2026}

\usepackage[utf8]{inputenc} % allow utf-8 input
\usepackage[T1]{fontenc}    % use 8-bit T1 fonts
\usepackage{hyperref}       % hyperlinks
\usepackage{url}            % simple URL typesetting
\usepackage{booktabs}       % professional-quality tables
\usepackage{amsfonts}       % blackboard math symbols
\usepackage{nicefrac}       % compact symbols for 1/2, etc.
\usepackage{microtype}      % microtypography
\usepackage{xcolor}         % colors

\title{Plug-In Classification of Drift Functions in Diffusion Processes Using Neural Networks}

\author{%
}
\author{%
  Yuzhen Zhao \\
  Universit\'e Paris-Dauphine, PSL \\
  Chaire DIALog, Fondation du Risque \\
  Institut Louis Bachelier \\
  Paris, France \\
  \And
  Jiarong Fan \\
  LaMME, University of Paris-Saclay \\
  Evry, France \\
  \And
  Yating Liu \\
  CEREMADE, CNRS \\
  Universit\'e Paris-Dauphine, PSL \\
  Paris, France \\
}
\usepackage{microtype}
\usepackage{graphicx}
\usepackage{subcaption}
\usepackage{booktabs} 
\usepackage{hyperref}

\usepackage{amsmath}
\usepackage{amssymb}
\usepackage{mathtools}
\usepackage{amsthm}

\usepackage[capitalize,noabbrev]{cleveref}

\theoremstyle{plain}
\newtheorem{theorem}{Theorem}[section]
\newtheorem{proposition}[theorem]{Proposition}
\newtheorem{lemma}[theorem]{Lemma}

\theoremstyle{definition}

\newtheorem{assumption}[theorem]{Assumption}
\theoremstyle{remark}
\newtheorem{remark}[theorem]{Remark}

\usepackage{bbm}
\usepackage{comment}

\renewcommand{\d}{\mathrm{d}}

\newcommand{\R}{\mathbb{R}}
\newcommand{\PP}{\mathbb{P}}
\newcommand{\RD}{\mathbb{R}^{d}}
\newcommand{\RR}{\mathbb{R}}
\newcommand{\EE}{\mathbb{E}}

\newcommand{\calY}{\mathcal{Y}}
\newcommand{\calF}{\mathcal{F}}
\newcommand{\calG}{\mathcal{G}}
\newcommand{\calD}{\mathcal{D}}

\newcommand{\calR}{\mathcal{R}}

\newcommand{\calE}{\mathcal{E}}

\DeclareMathOperator*{\argmax}{argmax}
\DeclareMathOperator*{\argmin}{argmin}

\newcommand{\nnset}{\mathcal{F}(L, \mathbf{p}, s, F)}

\usepackage[textsize=tiny]{todonotes}
\allowdisplaybreaks

\begin{document}

\maketitle

\begin{abstract}
We study supervised multiclass classification for diffusion processes, where each class is characterized by a distinct drift function and trajectories are observed at discrete times. We first derive a multidimensional Bayes rule and then construct a plug-in classifier by estimating the class-specific drifts with neural networks. Under standard regularity assumptions, we establish convergence rates for the excess misclassification risk, making explicit the contributions of drift estimation, time discretization, and dimension. Our analysis also highlights the benefit of exploiting the diffusion structure: the drift is learned from all observed increments, leading to sharper guarantees than direct trajectory-based neural classifiers in the considered setting. Numerical experiments support the theory: the proposed method achieves better classification performance than \cite{Denis2024} in dimension one, remains effective in higher dimensions when the drift functions admit a compositional structure, and outperforms end-to-end neural classifiers trained directly on trajectories, as in \cite{bos2022}.
\end{abstract}

\section{Introduction}

In this paper, we study a supervised multiclass classification problem for the drift functions of a time-homogeneous diffusion process using a neural-network (NN) based plug-in classifier. Specifically, we consider a filtered probability space $(\Omega, \calF, (\calF_t)_{t\in[0,T]}, \PP)$, an $(\calF_t)_{t\in[0,T]}$-adapted $d$-dimensional standard Brownian motion  $B = (B_t)_{t \in [0, T]}$, and an $\mathbb{R}^d$-valued diffusion process $X=(X_t)_{t\in[0,T]}$ defined on $(\Omega, \calF, (\calF_t)_{t\in[0,T]}, \PP)$ and solving the stochastic differential equation (SDE)
\begin{equation}\label{eq:sde}
\d X_t = b_Y(X_t)\,\d t + \sigma(X_t)\,\d B_t,
\end{equation}
where $X_0$ is an $\RD$ valued random vector, and $Y$ is a discrete random variable taking values in $\mathcal{Y}\coloneqq\{1,\dots,K\}$ representing the class label. For each label $k\in\mathcal{Y}$, the drift $b_k : \mathbb{R}^d \to \mathbb{R}^d$ and the diffusion coefficient $\sigma : \mathbb{R}^d \to \mathbb{R}^{d\times d}$ are measurable. We assume that  $X_0$ and the Brownian motion $B=(B_t)_{t\in[0,T]}$ are both independent of the label $Y$. 

%\noindent\textbf{Data :}%$\bar{X}^{[k],n} = (\bar X^{[k],n}_t)_{t \in [0, T]}$, $1 \le n \le N$, 

For each fixed label $Y=k\in\mathcal{Y}$, we assume that $N_k$ independent and identically distributed (i.i.d.) sample paths are available and observed at high frequency, that is, at discrete time points on a fine temporal grid. This yields the dataset  $\mathcal{D}_{N_k}^{[k]}$ given by
\begin{align}
 \mathcal{D}_{N_k}^{[k]}
 =
 \Big\{
 \bar{X}^{[k],(n)}_{t_0:t_M}
 \!=\!
 \big(
\bar X^{[k],(n)}_{t_0}\!\!,
 \dots,
\bar X^{[k],(n)}_{t_M}
 \big),   
 1\!\le\! n\!\le\! {N_k}
 \Big\},\nonumber
\end{align}
where $t_m = m\Delta$ with time step $\Delta = \tfrac{T}{M} \to 0$ as $M \to \infty$. We denote by $\mathcal{D}_N = \bigcup_{k=1}^K \mathcal{D}_{N_k}^{[k]}$ the full training dataset, where $N=\sum_{k=1}^{K}N_k$.

The objective of this paper is to construct a classifier
$
g^{\mathcal{D}_N} : (\mathbb{R}^d)^{M+1} \to \mathcal{Y}
$
based on the dataset $\mathcal{D}_N$ such that, given a new observed trajectory
$
\bar Z = (Z_{t_0}, \dots, Z_{t_M}),
$
$g^{\mathcal{D}_N}(\bar Z)$ predicts the corresponding label of the underlying diffusion process $Z$.

%Since the label $Y$ affects the dynamics only through the drift function, we assume throughout the paper that the diffusion coefficient $\sigma$ is known and that the prior label probabilities
%$\mathfrak{p}_k \coloneqq \mathbb{P}(Y = k)$ are given. But the drift functions $b_k, k\in \calY$ are unknown and must be estimated from the data.  When this is not the case, $\sigma$ can be consistently estimated from high-frequency data \rr (see, e.g., \cite{<ref>}), \black and the class priors $\mathfrak{p}_k$ can be estimated by
%$
%\widehat{\mathfrak{p}}_k = N_k / N,
%$
%where $N_k$ denotes the number of training trajectories with label $k$.

%The analysis in this paper is carried out under the assumption that the observed trajectories are generated by an underlying diffusion process. 
The analysis in this paper is carried out under the assumption that the observed trajectories are generated by an underlying diffusion process, and the class label depends only on the drift function \(b_Y\), whereas the diffusion coefficient \(\sigma\) is assumed to be identical across classes. 
Testing whether a given observation is consistent with a diffusion model belongs to the literature on goodness-of-fit and specification testing for diffusion models; see, for instance, the works of Yacine Aït-Sahalia and his collaborators \cite{ait2002telling, MR2488349, AITSAHALIA2018233, MR2722466}, and the recent review \citet{lopezreview}. Moreover, for simplicity of notation, we assume throughout the paper that the prior class probabilities
$
\mathfrak{p}_k \coloneqq \mathbb{P}(Y = k)
$
are known. Otherwise, they can be estimated empirically by
$
\widehat{\mathfrak{p}}_k = \frac{N_k}{N},
$
where \(N_k\) is the number of training trajectories with label \(k\), and \(N\) is the total number of training trajectories.

%%%%%%%
%%%%%%%

%Since the class label $ Y $ influences the dynamics only through the drift function, we assume throughout the paper that the diffusion coefficient $ \sigma $ is known and that the prior class probabilities
%$ \mathfrak{p}_k \coloneqq \mathbb{P}(Y = k) $ are given. In contrast, the drift functions $ b_k $, $ k \in \mathcal Y $, are unknown and must be learned from the data.  If the diffusion coefficient $\sigma$ and the class priors $\mathfrak{p}_k$, $k\in\mathcal Y$, are unknown, then $\sigma$ can be consistently estimated from high-frequency observations (see, e.g., \citet{Comte2007}), and the class priors $\mathfrak{p}_k$ can be estimated empirically as
%$
%\widehat{\mathfrak{p}}_k = \tfrac{N_k}{N},
%$
%where $N_k$ denotes the number of training trajectories labeled $k$.

%Since the label $Y$ affects the dynamics only through the drift function, we assume throughout the paper that the diffusion coefficient $\sigma$ is known and that the prior label probabilities
%$\mathfrak{p}_k \coloneqq \mathbb{P}(Y=k)$ are given. When this is not the case, $\sigma$ can be consistently estimated from high-frequency data (see, e.g., \cite{<ref>}), and the class priors $\mathfrak{p}_k$ can be estimated by
%$\widehat{\mathfrak{p}}_k = N_k / N$, where $N_k$ denotes the number of training trajectories with label $k$.

%[describe sigma is know, $\mathfrak{p}_k$ is known]

%\medskip

%\noindent\textbf{Label :}

%$Y$ which is a r.v. taking values in $\calY\coloneqq\{1, ..., K\}$

%\medskip

%\noindent\textbf{Objective :}

%Construct a classifier 

%\medskip

%\subsection{Literature Review and Motivation}

\subsection{Literature Review and Motivation}

Diffusion processes of the form~\eqref{eq:sde} constitute a fundamental class of stochastic models with a wide range of applications in physics, biology, and mathematical finance (see, e.g., \citet{Gardiner2004, Bressloff2014, Karatzas1998}). More recently, they have gained renewed attention as the continuous-time theoretical foundation of generative diffusion models, which have achieved state-of-the-art performance in image generation and related tasks (see, e.g., \citet{song2021scorebased, dhariwal2021diffusion}). The supervised classification of diffusion paths, where class information is encoded through the drift function, has been investigated in \citet{Cadre2013, Gadat2020, Denis2020, Denis2024, denis2025empirical}. However, existing results are largely restricted to the one-dimensional setting, which limits their applicability in modern high-dimensional contexts. % such as diffusion-based generative models.

This paper extends this line of research to multidimensional diffusion processes. The proposed classifier is based on the nonparametric estimation of the drift function, a well-studied problem in the diffusion literature (see e.g. \citet{Hoffmann1999, Comte2020, Denis2021, zhao2025drift}). We note that drift estimation and classification are distinct tasks. Classification provides a statistical decision framework to map complex estimated functions, possibly learned by neural networks, onto a finite set of interpretable diffusion processes. This classification perspective not only facilitates theoretical analysis but also enhances interpretability and practical relevance in downstream applications.

Alternatively, one may treat the whole observed trajectory $
 \big(
\bar X^{[k],(n)}_{t_0}\!\!,
 \dots,
\bar X^{[k],(n)}_{t_M}
 \big)$  as input and train a trajectory-based classifier, such as a feedforward neural network (FNN) or an RNN (see, e.g., \citet{bos2022, chen2020generalization}). These methods are flexible but do not exploit the diffusion structure. This motivates a comparison with our structure-aware plug-in approach in the diffusion setting.

\subsection{Contribution and Organization of this Paper}

This paper studies supervised multiclass classification for multidimensional diffusion processes, where each class is characterized by a distinct drift function. On the theoretical side, we make the following contributions:
\begin{enumerate}
\item Proposition~\ref{prop:bayes-classifier} provides a characterization of the Bayes classifier in the multidimensional setting, extending the one-dimensional result of \citet[Proposition~1]{Denis2020}. This result also serves as the foundation for constructing plug-in classifiers based on estimated drift functions.
\item Theorem~\ref{thm:main-decomposition} establishes a decomposition of the excess classification risk, defined as the gap between the risk of an arbitrary classifier and that of the Bayes classifier, in terms of the time step size~$\Delta$ and the drift estimation error.
\item Theorem~\ref{thm:main-NN} derives convergence rates for the neural network based plug-in classifier using drift estimation results from \citet{zhao2025drift}.
\item  Appendix~\ref{app:trajectory-comparison} provides a theoretical comparison with FNN-based trajectory classifiers. Using the framework of \citet{bos2022}, we show that treating the whole path as input yields a convergence rate that depends on the number of observation times \(M\), highlighting the benefit of exploiting the diffusion structure. 
\end{enumerate}

On the numerical side, we consider two simulation examples. The first involves a diffusion with locally fluctuating drift functions, where we show that the proposed neural network based classifier significantly outperforms B-spline based methods,  as well as trajectory-based classifiers that take the entire observed path as input, such as FNNs, RNNs, TCNs, and Transformers, without exploiting the underlying SDE structure. 
Moreover, the empirical convergence behavior agrees with Theorem~\ref{thm:main-NN}. In this example, the compositional structure of the drift functions ensures that the dimension affects only the constants, and not the exponent of the convergence rate, and this is reflected in the observed rates across multidimensional settings. 
%under the compositional structure of the drift functions, the dimension affects the constants but not the convergence-rate exponent, as reflected by the observed rates in multidimensional settings. 
The second experiment revisits the example of \citet{Denis2024}, where our method achieves performance comparable to both their classifier and the Bayes benchmark.

The paper is organized as follows. All notations used in this paper are collected in Section~\ref{sec:notations}. Section~\ref{sec:main-theory} introduces the construction of the neural network based plug-in classifier, relying on the Bayes characterization established in Proposition~\ref{prop:bayes-classifier}. Theoretical convergence guarantees for this classifier are established in Theorems~\ref{thm:main-decomposition} and~\ref{thm:main-NN}. Section~\ref{sec:numerical-part} presents the numerical experiments. Finally, Section~\ref{sec:conclusion} presents the conclusion of this paper. The appendix contains detailed proofs and implementation details. % Section~\ref{sec:ProofSketch} provides proof sketches for the main theoretical results. %

\subsection{Notation}\label{sec:notations}

%[TO DO TO DO]

We denote by $\mathcal{C}([0,T], \mathbb{R}^d)$ the space of continuous functions from $[0,T]$ to $\mathbb{R}^d$. For a function $f:D\rightarrow \RR^d$, we write $\Vert f\Vert _{\sup}=\sup_{x\in D}|f(x)|$.  The prior probability of class $k\in\mathcal{Y}$ is denoted by
$\mathfrak{p}_k \coloneqq \mathbb{P}(Y = k)$. For a random variable $X$, $\Vert X\Vert_p$ denotes the $L^p$-norm of $X$, that is, $\Vert X\Vert_p=\{\EE [|X|^p]\}^{\frac{1}{p}}.$ 

For a vector or matrix $W$, we write $|W|$ for the Euclidean norm when $W$ is a vector and for the Frobenius norm when $W$ is a matrix. The notation $|W|_\infty$ denotes the maximum-entry norm, $|W|_0$ denotes the number of nonzero entries of $W$, and $|W|_{\text{op}}$ denotes the operator norm. 

For $\beta \in \mathbb{R}$, $\lfloor \beta \rfloor$ denotes the largest integer strictly smaller than $\beta$.  
For two sequences $(a_N)$ and $(b_N)$, we write $a_N \lesssim b_N$ if there exists a constant $C>0$ such that $a_N \le C b_N$ for all $N$, and we write $a_N \asymp b_N$ if both $a_N \lesssim b_N$ and $b_N \lesssim a_N$ hold.

Finally, the index $k = 1,\dots,K$ denotes the class label in the classification problem, the index $m = 0,\dots,M-1$ refers to discrete time steps on the observation grid, and $n = 1,\dots,N$ indexes the sample paths in the training dataset. %We use the superscript $[k]$ to indicate class labels and the superscript $(n)$ to denote individual sample paths in the training set. 
Throughout the paper, $\mathfrak{C}$ denotes a generic positive constant depending only on the model parameters $(d,T,b_1,\ldots,b_K,\sigma,\|X_0\|_4)$ and more generally, $C_{\lambda_1,\ldots,\lambda_p}$ denotes a positive constant depending on the parameters $\lambda_1,\ldots,\lambda_p$. The value of $\mathfrak{C}$ and $C_{\lambda_1,\ldots,\lambda_p}$ may  vary from line to line.

\section{Construction and Convergence Analysis of the Neural Network Based Plug-In Classifier}\label{sec:main-theory}

Throughout the paper, we work under the following assumptions.
\begin{assumption}\label{assump:Lipschitz}\begin{enumerate}
    \item[$(a)$] $\min_{k\in\calY}\mathfrak{p}_k>0$, $\EE[|X_0|^4]<+\infty$; 
%    \item[$(b)$]  
    \item [$(b)$] The coefficient functions $b_k, \,k\in\calY$ and $\sigma$ are globally Lipschitz continuous, that is, there exist constants $L_b, L_\sigma > 0$ such that for every $ x, y \in \mathbb{R}^d,$
    \[
        \max_{k\in\calY} |b_k(x)-b_k(y)|\leq L_b|x-y|,\qquad |\sigma(x)-\sigma(y)|\leq L_\sigma  |x-y|;\nonumber
    \]
\item[$(c)$] For every $x\in\RD$, the matrix $\sigma(x)$ is invertible.  
\end{enumerate}
\end{assumption}
%The coefficient functions $b_k, \,k\in\calY$ and $\sigma$ are globally Lipschitz continuous, that is, there exist constants $L_b, L_{\sigma} > 0$ such that for all $x, y \in \mathbb{R}^d$, for every $k\in\calY$
%\begin{align*}
%|b_k(x)-b_k(y)| \leq L_b |x-y|, \quad |\sigma(x)-\sigma(y)| \leq L_{\sigma} |x-y|.
%\end{align*} Moreover, for every $x\in\RD$, the matrix $\sigma(x)$ is invertible.  
%\end{assumption}

%\begin{assumption}
%\begin{manualtheorem}{I}
%\label{assum:lip} The initial random variable $X_0$ satisfies $\EE[|X_0|^2]<+\infty$. 
%The coefficient functions $b$ and $\sigma$ are globally Lipschitz continuous; that is, there exist constants $L_b, L_{\sigma} > 0$ such that for all $x, y \in \mathbb{R}^d$,
%\begin{align*}
%|b(x)-b(y)| \leq L_b |x-y|, \quad |\sigma(x)-\sigma(y)| \leq L_{\sigma} |x-y|.
%\end{align*}
%

\begin{assumption}\label{assump:Novikov} (Novikov's condition) 
  $\EE\big[\exp \big(\tfrac{1}{2}\int_0^T |\sigma^{-1}(X_s)\,b_k(X_s)|^2\,\d s\big)\big]<+\infty$,  $k\in\calY$.
\end{assumption}

\begin{assumption}\label{assump:main-theorem}
For every $x \in \mathbb{R}^d$, the matrix $a(x) \coloneqq \sigma\sigma^{\top}(x)$ is invertible, and $a^{-1}$ is globally Lipschitz continuous with respect to the operator norm $|\cdot|_{\mathrm{op}}$, with Lipschitz constant $L_{a^{-1}}$. Moreover, there exists a constant $\Lambda > 0$ such that $\max(|\sigma(x)|,|a(x)^{-1}|_{\mathrm{op}}) \leq \Lambda, \,x\in\RD$.
\end{assumption}

For clarity of exposition, we assume in what follows that the diffusion coefficient $\sigma$ is known, which allows us to present the main ideas without overloading the notation. The case where $\sigma$ is unknown can be handled by replacing $\sigma$ with an estimator $\widehat{\sigma}$, and is discussed in Appendix \ref{app:sigma-unknown}.

\subsection{Bayes classifier}
Consider a diffusion process $(X_t)_{t\in[0,T]}$ that is a solution to~\eqref{eq:sde}
and is observed continuously in time. For a given classifier
$
g : \mathcal{C}([0,T],\mathbb{R}^d) \to \mathcal{Y},
$
its performance is measured by the misclassification risk
\begin{equation}\label{eq:def-misclassification-risk}
\mathcal{R}(g) \coloneqq \mathbb{P}\big( g(X) \neq Y \big).
\end{equation}
The Bayes classifier $g^*$ is defined as any classifier minimizing the misclassification risk over the class
\[
\mathbb{G} \coloneqq \big\{ g : \mathcal{C}([0,T], \mathbb{R}^d) \to \mathcal{Y} \;\text{measurable} \big\},
\]
that is, 
%\begin{equation}\label{eq:def-classifier-bayes}
$\displaystyle g^* \in \argmin_{g \in \mathbb{G}} \mathcal{R}(g).$   
%\end{equation}
Moreover, $g^*$ admits the characterization $\displaystyle g^*(X) \in \argmax_{k \in \mathcal{Y}} \pi_k^*(X)$ (see, e.g., \citet[Section 2.4]{Hastie2009elements}),  where 
\begin{equation}\label{eq:def-pi}
\pi_k^*(X) = \mathbb{P}(\,Y = k \mid X\,).  
\end{equation}
The following proposition, whose proof is provided in Appendix~\ref{sec:app-a}, extends \citet[Proposition~1]{Denis2020} to a high-dimensional setting.

\begin{proposition}\label{prop:bayes-classifier} 

Assume that Assumptions~\ref{assump:Lipschitz} and~\ref{assump:Novikov} hold.
For each $k \in \mathcal{Y}$, define
\begin{align}\label{eq:def-F}
F_k^*(X)
\coloneqq\;
&\int_{0}^{T} b_k(X_s)^{\top} \big(\sigma\sigma^{\top}\big)^{-1}(X_s)\,\d X_s -\frac{1}{2}\int_{0}^{T} \big\lvert \sigma^{-1}(X_s) b_k(X_s) \big\rvert^2 \,\d s .
\end{align}
Then, for each $k \in \mathcal{Y}$, the conditional probability $\pi_k^*$ defined in \eqref{eq:def-pi} satisfies
\begin{equation}\label{eq:pi-equality}
\pi_k^*(X) = \phi_k\big( F^*(X) \big), \quad \PP-\text{a.s.,}
\end{equation}
where $F^* = (F_1^*, \dots, F_K^*)$ and
\begin{equation}\label{eq:def-phi}
\phi_k(x_1,\dots,x_K)
\coloneqq
\frac{\mathfrak{p}_k \exp(x_k)}{\sum_{j=1}^{K} \mathfrak{p}_j \exp(x_j)}, \;1\leq k\leq K  
\end{equation}
denote the softmax functions with prior weights $\mathfrak{p}_k, \,1\leq k\leq K$.
\end{proposition}
Taking the Bayes classifier $g^*$ as a reference, the performance of any classifier $g\in\mathbb{G}$ is evaluated using the excess classification risk 
\begin{equation}\label{eq:def-excess classification risk}
   \calR(g)-\calR(g^*). 
\end{equation} This criterion will be used to assess the performance of our neural network based classifier in the sequel.

%Based on the Bayes classifier $g^*$ defined in \eqref{eq:def-classifier-bayes}, the performance of any classifier $g$ is evaluated using the excess classification risk, defined as $\calR(g)-\calR(g^*)$. This criterion is also used to assess the performance of our neural network based classifier in the sequel. 

%\rr the proposed classifier in this paper the preformance is evaluated by using the excess classification risk ... \black

\subsection{Construction of the NN-based plug-in classifier}\label{subsec:construction-nn-classifier}

Using the characterization of the Bayes classifier established in Proposition~\ref{prop:bayes-classifier}, we describe the construction of the NN-based plug-in classifier in two steps.

\paragraph{Step 1: Time discretization of $F_k^*$ and plug-in classification.}
We begin by introducing a time-discretized version of the functional $F_k^*$ defined in \eqref{eq:def-F}. Given a discretely observed trajectory $(X_{t_0}, \dots, X_{t_{M}})$ of the diffusion process $X$, we define 
\begin{align}\label{eq:F-bar}
\bar{F}_k(X)
\coloneqq
&\sum_{m=0}^{M-1} b_k(X_{t_m})^{\top} \big(\sigma\sigma^{\top}\big)^{-1}(X_{t_m})\,(X_{t_{m+1}}\!\!-X_{t_m}) -\frac{\Delta}{2}\sum_{m=0}^{M-1} \big\lvert \sigma^{-1} (X_{t_m})\,b_k(X_{t_m}) \big\rvert^2.
\end{align}
Since the drift functions \( b_k \) are unknown and must be estimated, we assume that, for each \( k \in \mathcal Y \), a nonparametric estimator \( \widehat b_k \) of \( b_k \) is available. Plugging these estimators into \eqref{eq:F-bar} yields the following implementable score function:
\begin{align}\label{eq:F-hat}
\!\!\!\widehat{F}_k(X)
\coloneqq
\sum_{m=0}^{M-1} \widehat{b}_k(X_{t_m})^{\top} \big(\sigma\sigma^{\top}\big)^{-1}(X_{t_m})\,(X_{t_{m+1}}\!\!-X_{t_m}) -\frac{\Delta}{2}\sum_{m=0}^{M-1} \big\lvert \sigma^{-1} (X_{t_m})\,\widehat{b}_k(X_{t_m}) \big\rvert^2.
\end{align}
The associated probabilities and the resulting plug-in classifier are then defined by
\begin{equation}\label{eq:def-pi-hat-g-hat}
\widehat\pi_k(X) = \phi_k\big( \widehat{F}(X) \big), \quad \widehat g(X)\in\argmax_{k\in\mathcal{Y}} \widehat \pi_k(X),
\end{equation}
where $\widehat{F}(X)=(\widehat{F}_1(X),\dots, \widehat{F}_K(X))$ and $\phi_k,\,k\in\calY$ are softmax functions defined by \eqref{eq:def-phi}. 
%The resulting plug-in classifier is then defined by
%\begin{equation}\label{eq:def-g-hat}
%\end{equation}

\paragraph{Step 2: Drift estimation.}  
We now describe the construction of the nonparametric estimators \( \widehat b_k \) of the drift functions \( b_k \) for each class \( k \in \mathcal Y \).
For a fixed label \( k \), we estimate \( b_k \) on a compact set $\mathcal{K}$ using the class-specific training dataset \( \mathcal D_{N_k}^{[k]} \) and neural networks defined further in \eqref{eq:sparseNN}. The compact set \(\mathcal K\) can be selected based on the distribution of the training data, for instance via sample coverage or a sample-splitting procedure. 
Following the approach of \citet{zhao2025drift}, the drift function is estimated component-wise using feedforward neural networks.

Let $\mathcal{F}_{L,\mathbf{p}}$ denote the class of feedforward neural networks with $L$ hidden layers and layer widths $\mathbf{p}=(p_0,p_1,\dots,p_{L+1})\in\mathbb{N}^{L+2}$, where $p_0=d$ is the input dimension and $p_{L+1}=1$ is the output dimension. Each function $f\in\mathcal{F}_{L,\mathbf{p}}$ maps $\mathbb{R}^d$ to $\mathbb{R}$ and admits the representation
\begin{equation}\label{eq:NN}
f(x)
=
W_L \sigma_{\mathbf v_L}
W_{L-1} \sigma_{\mathbf v_{L-1}}
\cdots
W_1 \sigma_{\mathbf v_1}
W_0 x,
\end{equation}
where $W_j$ is a $p_{j+1}\times p_j$ weight matrix. The activation function $\sigma(x)=\max(x,0)$ is the ReLU function applied component-wise, and for $\mathbf v=(v_1,\dots,v_r)\in\mathbb{R}^r$, the shifted ReLU operator $\sigma_{\mathbf v}:\mathbb{R}^r\to\mathbb{R}^r$ is defined by
\[
\sigma_{\mathbf v}\big((y_1,\dots,y_r)^\top\big)
=
\big(\sigma(y_1-v_1),\dots,\sigma(y_r-v_r)\big)^\top.
\]
In this paper, we consider sparse neural networks with sparsity level $s$ and impose a uniform boundedness constraint with constant $F>0$. We define the class of admissible neural network estimators as
\begin{align}%\label{eq:def-nn-set}
&\mathcal{F}(L, \mathbf{p}, s, F) 
:= \Big\{ f\mathbbm{1}_{\mathcal{K}} \::\: f\in \mathcal{F}_{L, \mathbf{p}} \text{ such that }  \max_{j=0,\dots,L} \big( \|W_j\|_{\infty} \vee |\mathbf{v}_j|_{\infty} \big) \le 1,\; \big\| f \big\|_{\infty} \le F,\nonumber \\
&\hspace{3cm}   \text{ and } \sum_{j=0}^{L} \big( \|W_j\|_0 + |\mathbf{v}_j|_0 \big) \le s \Big\}.
\label{eq:sparseNN}
\end{align}
%\rr [yating: we need to explain here why we have $\mathbbm{1}_{\rr [0,1]^d}$] [how we choose the parameter F in the numerical example ? ] \black
For every $i\in\{1, ..., d\}$, the estimator $\widehat b_k^i$ of the $i$-th component of $b_k$ is obtained by minimizing the empirical loss
\begin{equation}\label{eq:train-nn-loss}
\mathcal{Q}^i_{\mathcal D_{N_k}^{[k]}}(f)
\coloneqq\!
\frac{1}{N_kM}
\sum_{n=1}^{N_k}
\sum_{m=0}^{M-1}
\left(
Y^{[k],(n),i}_{t_m}
-
f\big(\bar X^{[k],(n)}_{t_m}\big)
\right)^2
\end{equation}
over the class $\mathcal{F}(L,\mathbf{p},s,F)$, where
\begin{equation}\label{eq:defY}
Y^{[k],(n),i}_{t_m}
\coloneqq
\frac{1}{\Delta}
\left(
\bar X^{[k],(n),\,i}_{t_{m+1}} - \bar X^{[k],(n),\,i}_{t_m}
\right)
\end{equation}
denotes the discrete-time increment of the $i$-th coordinate of the process
$\bar X^{[k],(n)} = (\bar X^{[k],(n)}_{t_0}, \dots, \bar X^{[k],(n)}_{t_M})$.

\subsection{Main Theorems}

We now present the main theoretical results of this paper, which establish convergence rates for the proposed neural network–based plug-in classifier \( \widehat g \) in terms of excess classification risk \eqref{eq:def-excess classification risk}.
For each class \(k\in\mathcal Y\), let
\begin{equation}\label{eq:def-estimation-error-all}
\mathcal{E}(\widehat b_k, b_k)
\coloneqq
\EE \!\left[
\frac{1}{M}\sum_{m=0}^{M-1}
\Big(\widehat b_k(X_{t_m})-b_k(X_{t_m})\Big)^2
\right]
\end{equation}
denote the global drift estimation error, where \(X=(X_t)_{t\in[0,T]}\) is an independent trajectory solving \eqref{eq:sde} and independent of the training dataset \(\mathcal D_N\). Note that in the definition of \(\mathcal{E}(\widehat b_k,b_k)\), the test process \(X\) is not conditioned on any specific class label. % in other words, \(X\) may originate from any class in \(\mathcal Y\).
Theorem~\ref{thm:main-decomposition} establishes a decomposition of the excess classification risk
\(\mathcal R(\widehat g)-\mathcal R(g^*)\) in terms of the drift estimation errors \(\mathcal{E}(\widehat b_k,b_k)\) and the time discretization step \(\Delta\).

In practice, however, each drift \(b_k\) is estimated solely from training trajectories belonging to class \(k\).
Accordingly, we relate the global estimation error \(\mathcal{E}(\widehat b_k,b_k)\) to the class-conditional estimation error
\begin{equation}\label{eq:def-estimation-error-class-k}
\!\!\mathcal{E}_j(\widehat b_k,b_k)
\coloneqq
\EE_j \!\left[
\frac{1}{M}\!\sum_{m=0}^{M-1}\!
\Big(\widehat b_k(X_{t_m})-b_k(X_{t_m})\Big)^2
\right],\!
\end{equation}
where $\EE_j$ denotes the expectation under $\PP_j$, and under $\PP_j$, the trajectory \(X=(X_t)_{t\in[0,T]}\) has label \(Y=j\), see further Lemma~\ref{lem:relation-E-with-Ek}.
Then Theorem~\ref{thm:main-NN} further specializes this relation to neural network drift estimators constructed as described in Section~\ref{subsec:construction-nn-classifier}, yielding explicit convergence rates for the classification error. Proofs of Theorems \ref{thm:main-decomposition} and \ref{thm:main-NN} are given in Appendix~\ref{sec:app-a}.

%Proof sketches for  Section~\ref{sec:ProofSketch}, and  are deferred to 

\begin{theorem}\label{thm:main-decomposition}
Assume that Assumptions~\ref{assump:Lipschitz}, \ref{assump:Novikov} and \ref{assump:main-theorem} hold.
For each \( k \in \mathcal Y \), let \( \widehat b_k \) be an estimator of the drift function \( b_k \) such that 
$\sup_{x\in\RD} |\widehat b_k(x)|\leq \widehat b_{\max}$ for some constant $\widehat b_{\max}>0$. Let \( \widehat g \) denote the plug-in classifier defined in \eqref{eq:def-pi-hat-g-hat} associated with \( \widehat b_k, \,k\in\calY \). 
Then there exists a constant \( C_{\Lambda, \mathfrak{C}, b_{\max}} > 0 \), such that
\begin{equation}\label{eq:inequality-in-main-thm}
    \calR(\widehat{g})-\mathcal{R}(g^*)\leq K^2 C_{\Lambda, \mathfrak{C}, \widehat b_{\max}}\big(\sqrt{\Delta} + \max_{k\in\calY}\calE(\widehat b_k, b_k)^{\frac{1}{2}}\big).
\end{equation}
\end{theorem}
We now provide an upper bound on the excess classification risk when the estimators
\(\widehat b_k\), \(k\in\mathcal Y\), are constructed using neural networks.
To simplify the notation, we assume that the sample sizes satisfy \(N_k=\tfrac{N}{K}\) for all
\(k\in\mathcal Y\).
We first introduce the error arising from the neural network training procedure.
Specifically, given the empirical loss $\mathcal{Q}^i_{\mathcal D_{N_k}^{[k]}}(\widehat{f}_i)$ in \eqref{eq:train-nn-loss} and its exact minimizer
over the function class \(\mathcal{F}(L,\mathbf{p},s,F)\), we define
\begin{align}\label{eq:def-psi-i}
&\Psi^{\mathcal{F}, [k],i}\big(\widehat f_i\big)
\coloneqq \EE\Big[
\mathcal{Q}^i_{\mathcal D_{N_k}^{[k]}}(\widehat f_i\,)-\inf_{f\in\mathcal{F}(L,\mathbf{p},s,F)} \mathcal{Q}^i_{\mathcal D_{N_k}^{[k]}}(f)
\Big]\:\text{  and  }\:\Psi^{\mathcal{F}, [k]}\big(\widehat f\big)\coloneqq \max_{1\leq i\leq d}\Psi^{\mathcal{F}, [k],i}\big(\widehat f_i\big) \nonumber
\end{align}
with $\widehat f=(\widehat f_1, \dots, \widehat f_d)$. 
Let 
\begin{align}
\mathcal{C}_{r}^{\beta}(D, \widetilde{K})\!=\!\bigg\{f:D\subset \RR^r\rightarrow \RR :\!\!
\!\!\!\sum_{\boldsymbol{\alpha}:|\boldsymbol{\alpha}|<\beta}\!\!\Vert \partial^{\boldsymbol{\alpha}}f\Vert_{\sup}+\!\!\!\!\!\sum_{\boldsymbol{\alpha}:|\boldsymbol{\alpha}|=\lfloor\beta\rfloor}\sup_{x,y\in D, x\neq y}\!\!\frac{|\partial^{\boldsymbol{\alpha}}f(x)-\partial^{\boldsymbol{\alpha}}f(y)|}{|x-y|_{\infty}^{\beta - \lfloor\beta\rfloor}}\leq \widetilde{K}\bigg\},\nonumber
\end{align}
and let $\mathcal{G}(q,\mathbf{d}, \mathbf{t},\boldsymbol{\beta}, \widetilde{K})$ be the function space defined in \citet{SchmidtHieber2020}:
\begin{align}
&\mathcal{G}(q,\mathbf{d}, \mathbf{t},\boldsymbol{\beta}, \widetilde{K})\coloneqq \big\{ f=g_q\circ \dots g_0: \:  g_i=(g_{ij})_j:[a_i, b_i]^{d_i}\rightarrow [a_{i+1}, b_{i+1}]^{d_{i+1}}, \nonumber\\
&\hspace{3cm}  g_{ij}\in\mathcal{C}_{t_i}^{\beta_i}([a_i, b_i]^{t_i}, \widetilde{K}), \text{ for some }|a_i|, |b_i|\leq \widetilde{K}\big\}.\nonumber
\end{align}
with $\mathbf{d}\coloneqq (d_0, \dots, d_{q+1})$, $\mathbf{t}\coloneqq(t_0, \dots, t_q)$, $\boldsymbol{\beta}\coloneqq (\beta_0,\dots, \beta_q)$.
Define $\beta_i^*\coloneqq \beta_i\prod_{l=i+1}^{q}(\beta_l\wedge 1)$ and 
\begin{equation}\label{eq:def-phi-N}
    \phi_{N}\coloneqq \max_{0\leq i \leq q}N^{-\frac{2\beta_i^*}{2\beta_{i}^*+t_i}}.
\end{equation}
\begin{assumption}
\label{assum:NN-parameter}
$F\geq \max(\sup_{x\in\mathcal{K}}|b(x)|,1)$, $L\geq1$, $s\geq2$, $N\geq2K$ and $\Delta \leq 1$. %\Yating{[Add other constant assumptions here if needed]}
\end{assumption}
\begin{theorem}\label{thm:main-NN}
Assume that \ref{assump:Lipschitz}, \ref{assump:Novikov}, \ref{assump:main-theorem} and \ref{assum:NN-parameter} hold. Assume moreover that, for every $k\in\calY$, %$b_k$ has a compact support $\mathcal{K}$ and satisfies 
$b_k\in \mathcal{G}(q,\mathbf{d}, \mathbf{t},\boldsymbol{\beta}, \widetilde{K})$, %Moreover, assume 
and that the neural network function class  $\nnset$ satisfies 
\begin{enumerate}
    \item[$\mathrm{(i)}$] $F\geq \max(K,1)$, $L\asymp \log_2N$, %$\sum_{i=0}^q \log_2(4t_i \vee 4 \beta_i)\log_2 n\leq L \lesssim n\phi_n $,
    \item[$\mathrm{(ii)}$] $N\phi_N\lesssim \min_{i=1, ..., L}p_i$, $s\asymp N\phi_N \log N$.
\end{enumerate}
Then there exists a constant $\widetilde{C}$  depending  on $q, \mathbf{d}, \mathbf{t}, \boldsymbol{\beta}, F, \mathfrak{C}, \mathcal{K}$ such that if \[\Delta \lesssim \phi_N\log^3N \text{ and }\max_{k\in\calY}\Psi^{\calF}(\widehat b_k)\leq C \phi_N\log^3N,\] then for any $\varepsilon\in(0,\frac{1}{4}]$, it holds %$\calR (\hat f, f_0)\leq \widetilde{C} \phi_N \log^3 N$.
\begin{equation}
    \calR(\widehat{g})-\mathcal{R}(g^*)\leq C_{K,\Lambda, \mathfrak{C}, \widetilde{C},\mathcal{K}, \varepsilon}\Big(\sqrt{\Delta} + \phi_N^{\frac{1}{2}-\varepsilon} \log^{\frac{3}{2}-
\varepsilon} N\Big).\nonumber
\end{equation}
%where the constant $C_{K,\Lambda, \mathfrak{C}, \widetilde{C},\mathcal{K}, \varepsilon}$
\end{theorem}
\begin{remark}[Compact support assumption]
The assumption in Theorem~\ref{thm:main-NN} that each $b_k$ is defined on a compact support through the definition of
$\mathcal{G}(q,\mathbf{d},\mathbf{t},\boldsymbol{\beta},\widetilde K)$ is not restrictive in practice.
Indeed, the drift estimator $\widehat b_k$ can only be learned accurately on regions where data are observed, which are bounded with high probability. More precisely, Assumption~\ref{assump:Lipschitz} implies that
$\sup_{t \in [0,T]} \EE[|X_t|^4] < \infty$ (see e.g. \citet[Proposition 7.2]{pages2018numerical}).
A simple application of Markov's inequality yields
\[ \forall \, t\in [0,T], \quad 
\PP\big(|X_t| \ge R\big)
\;\le\;
\frac{\sup_{t\in[0,T]}\EE[|X_t|^4]}{R^4}.
\]
%which converges to zero as $R \to \infty$.
Hence, for any prescribed accuracy level $\varepsilon > 0$, one can choose $R$ sufficiently large such that
$\PP(|X_t| \ge R) \le \varepsilon$.
Consequently, one may assume without loss of practical generality that the true drift function $b$ is also defined
on a compact set covering the observed data, without significantly affecting the statistical guarantees of the classifier.
%In contrast, for the practical success of the NN-based plug-in classifier, it is more important to have sufficient data in regions where the drift functions $b_k$ differ, as these regions are the ones that drive class separability.
%In contrast, the practical success of the NN-based plug-in classifier primarily depends on having sufficient data in regions of the state space where the drift functions $b_k$ differ, 
In contrast, for the practical success of the NN-based plug-in classifier, it is more important to have sufficient data in regions where the drift functions $b_k$ differ, since these regions determine the discriminative signal between classes.  %as these regions are the ones that drive class separability.
\end{remark}

%\rr     add a remark to say $b$ support compact is only a technical condition but not a real condition \black \rr if $b_k$ are not support compact, [add something like the compact set $\mathcal{K}$ contains all points of the difference of $b_k$ something like this] \black

%[TO DO]

%[TO DO]

%[TO DO]

%[TO DO]

%[TO DO]

%\rr corollary of aaai here 
%\begin{corollary}\label{cor:cor-of-thm-main} Suppose that Assumptions \ref{assum:lip} and \ref{assum:NN-parameter} hold.
%Assume moreover that $f_0\in \mathcal{G}(q,\mathbf{d}, \mathbf{t},\boldsymbol{\beta}, K)$ and the neural network estimators set $\nnset$ satisfies %$\hat{f}\in\nnset$ with
%\begin{enumerate}
%    \item[$\mathrm{(i)}$] $F\geq \max(K,1)$, $L\asymp \log_2N$ %$\sum_{i=0}^q \log_2(4t_i \vee 4 \beta_i)\log_2 n\leq L \lesssim n\phi_n $,
%    \item[$\mathrm{(ii)}$] $N\phi_N\lesssim \min_{i=1, ..., L}p_i$, $s\asymp N\phi_N \log N$.
%\end{enumerate}
%Then there exists a constant $C$  depending  on $q, \mathbf{d}, \mathbf{t}, \boldsymbol{\beta}, F, C_b, C_\sigma, L_b, L_\sigma, T$ such that if \[\Delta \lesssim \phi_N\log^3N \text{ and }\Psi^{\calF}(\hat f)\leq C \phi_N\log^3N,\] then $\calR (\hat f, f_0)\leq C \phi_N \log^3 N$.
%\end{corollary}
%\black

\section{Numerical Experiments}\label{sec:numerical-part}

We illustrate the performance of the NN-based classifier defined in \eqref{eq:def-pi-hat-g-hat} through two numerical experiments. 

In Section~\ref{sec:double-layer-example}, we consider  double-layer potential drifts inspired by \citet{zhao2025drift}. In the one-dimensional setting, we demonstrate that the proposed NN-based plug-in classifier outperforms the B-spline–based plug-in classifier introduced in \citet{Denis2024}. We also compare our approach with a direct neural network classifier that ignores the underlying SDE structure and learns class labels directly from the observed data, and we show that explicitly exploiting the diffusion structure leads to significantly improved classification performance. Furthermore, we demonstrate that the proposed NN-based plug-in classifier remains computationally tractable in higher dimensions ($d=2,5,10,50$) and achieves favorable convergence behavior.

In Section~\ref{sec:cos-example}, we study the example originally introduced in \citet{Denis2024}. In this case, we show that our NN-based plug-in classifier attains performance comparable to the method proposed therein, which is known to be close to the Bayes classifier.

\subsection{Example 1 : Double-Layer Potential Drifts}\label{sec:double-layer-example}
%\rr [TO DO I think the example in AAAI do not satisfy the Novikov condition if we do not assume that $b$ has a compact support. ] \black

\textbf{Experimental setup.} 
In this section, we consider a three-class classification problem, where each class is characterized by a distinct drift function defined as
\begin{equation}\label{eq:numerical-def-b_k}
   b_k(x)\coloneqq -x + \phi \big({\theta(s(x)+\alpha_k)}\big) \mathbf{1}_{d}, \;k=1,2,3,
\end{equation}
where $\mathbf{1}_{d}=(1,..., 1)^{\top}$, $s(x)=\tfrac{1}{d}\sum_{i=1}^dx_i$,   $(\alpha_1, \alpha_2, \alpha_3)=(0,1,-1)$ controls the separation between classes, and $\theta=5$ controls the intensity of the fluctuations.  The initial condition $X_0$ follows a standard normal distribution $\mathcal{N}(0, I_d)$, and the diffusion coefficient $\sigma(x)$ is set to the identity matrix. Figure~\ref{fig:true_bk_and_path} (left) illustrates the drift functions $b_k(x)$ for $k=1,2,3$, and Figure~\ref{fig:true_bk_and_path} (middle) shows representative sample paths for each class in the one-dimensional setting.

\begin{figure}[ht]
%  \vskip 0.2in
  \begin{center}
     \begin{subfigure}{0.3\linewidth}
        \includegraphics[width=\linewidth]{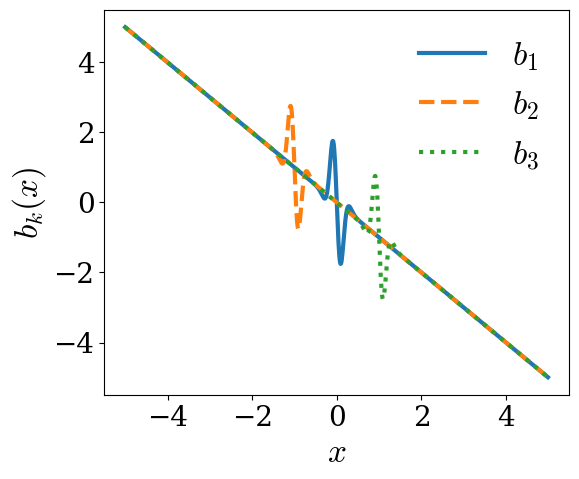}
  %      \caption{True $b_k,\, k\in\{1,2,3\}$}%\tag{1a}
 %       \label{fig:true_bk}
    \end{subfigure}
    \hfill
    \begin{subfigure}{0.3\linewidth}        \includegraphics[width=\linewidth]{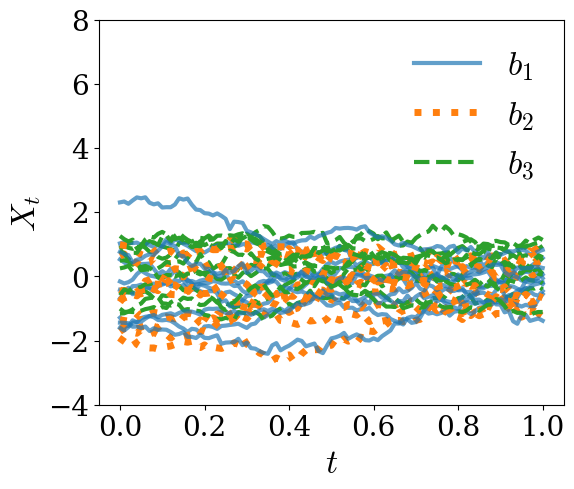}
 %      \caption{Sample paths for each class }
 %       \label{fig:true_bk-path}
    \end{subfigure}\hfill
    \begin{subfigure}{0.3\linewidth}        \includegraphics[width=\linewidth]{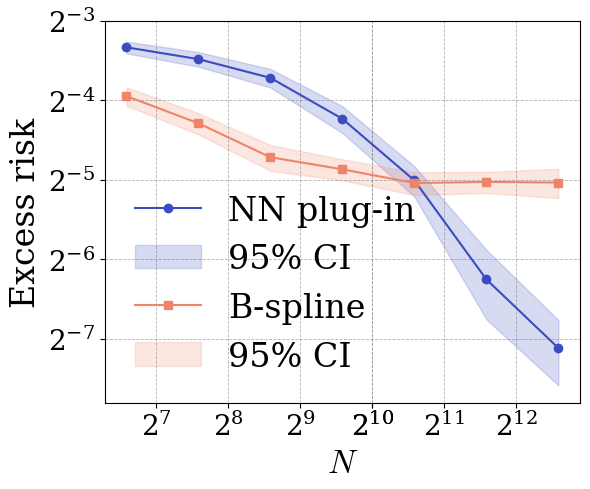}
 %      \caption{Sample paths for each class }
 %       \label{fig:true_bk-path}
    \end{subfigure}
% \caption{Figure 1}
\caption{True drift functions $b_k$, $k=1,2,3$ (left), sample paths from each class (middle), and comparison between the NN-based plug-in classifier and B-spline-based plug-in classifiers (right).}
   \label{fig:true_bk_and_path}
  \end{center}
\end{figure}

We fix the time horizon to $T = 1$ and the time step to $\Delta = 0.01$. Experiments are conducted for dimensions $d \in  \{1, 2, 5, 10, 50\}$, training sample sizes $ N_k \in \{ 2^5, 2^6, \cdots, 2^{12}\}$ for every label class $k\in\calY$, and a test sample size fixed at $N'=3000$. 

For the neural network estimator for $b_k$ in the definition of the plug-in classifier,  we use the  network architectures $\mathbf{p}= (d, 16, 32, 32, 16, 1)$ inspired by \citet{zhao2025drift}. The number of nonzero parameters $s$ is controlled as a proportion $s_{\mathrm{ratio}}$ of the total number of parameters, with $s_{\text{ratio}} = 0.75$. Training is performed using the Adam optimizer with a learning rate of $10^{-3}$. The number of training epochs is selected via early stopping. %, which is triggered after 20 consecutive epochs without improvement.
Additional implementation details are provided in the Appendix \ref{app:Experimental-3.1}.  For the B-spline based classifier, we strictly follow the same configuration as described in \citet[Section 3]{Denis2024}. %, namely, \rr blablablabla.

%[TO DO]\black

%[TO DO] 

Finally, for each experimental configuration, results are averaged over 50 independent repetitions. We report the mean error together with the corresponding 95\% confidence intervals. To study the convergence rate of the excess classification risk defined in \eqref{eq:def-excess classification risk}, we use log$_2$--log$_2$ plots of the empirical excess risk as a function of the sample size. Specifically, the horizontal axis corresponds to $\log_2N$, where $N$ denotes the number of training trajectories for all class $k$. We assume balanced classes, that is $N_k=\tfrac{N}{K}$. The vertical axis corresponds to the $\log_2$ of the excess classification risk. All implementations are carried out in Python and PyTorch.

%We compare three classifiers: the proposed NN-based plug-in classifier, a B-spline–based plug-in classifier \cite{Denis2024}, and a direct neural network classifier that ignores the underlying SDE structure \cite{bos2022}.

%Finally, for each experimental configuration, results are averaged over 50 independent repetitions, and we report the mean error together with the corresponding 95\% confidence intervals. 

%For the study of convergence rate of the excess classification risk defined in \eqref{eq:def-excess classification risk}, we plug  log-log error of the empirical convergence rate of the  with respect to sample $N$ for each label, which means the x-axis is $\log_2(N_k)$, where $N_k$ is the number of sample of label $k$, we assume that for each class $N_k$ are the same. the y-axis is the $\log_2$ of the excess classification risk for different classifier in the comparation, NN-based (ours), B-splined based (see \cite{Denis2024}) and direct NN claissifer ignoring the underlying sde struction (see \cite{bos2022})

%All implementations are performed using Python and  PyTorch.

\textbf{NN-based vs.\ B-spline–based plug-in classifiers.}
Figure~\ref{fig:true_bk_and_path} (right) shows that, in the one-dimensional setting, the proposed NN-based plug-in classifier achieves a faster convergence rate of the excess classification risk defined in \eqref{eq:def-excess classification risk} as the training sample size $N$ increases. In particular, for small sample sizes $N$, B-spline–based classifiers outperform neural networks classifiers; when $N$ is sufficiently large, the excess risk of the NN-based method continues to decrease significantly, whereas the B-spline–based approach shows a clear saturation effect. Moreover, as noted in \citet{zhao2025drift}, B-spline–based estimators of the drift function become computationally expensive in high-dimensional settings, which limits their practical applicability beyond low dimensions.

\begin{comment}

\begin{figure}[ht]
\begin{center}
    \includegraphics[width=0.5\linewidth]{figures-final/nnvsbspline.png}
\caption{NN-based plug-in classifier vs. B-spline-based plug-in classifier.}
\label{fig:NN-basedVSB-spline-based}
\end{center}
\end{figure}

\end{comment}

% due to their lower variance. However, as $N$ grows, their approximation error dominates, preventing further improvement in classification performance, while the NN-based classifier continues to benefit from additional data.
%Moreover, as noted in \cite{zhao2025drift}, B-spline–based drift estimators become computationally prohibitive in high-dimensional settings, which substantially limits their practical applicability beyond low-dimensional problems.
%\paragraph{NN-based vs. B-spline-based plug-in classifiers}
%Figure~\ref{fig:NN-basedVSB-spline-based} shows that our approach achieves a faster convergence rate of the excess risk defined in \eqref{eq:def-excess classification risk} with respect to the training sample size $N$ in the one-dimensional setting, when $N$ is large enough.  B-spline–based estimators works well when $N$ is small but when $N$ is large, error of B-spline do not decrease in the figure well NN decrease significantly. 

\textbf{NN-based plug-in classifier vs. trajectory-based classifiers.}
Figure~\ref{fig:NN-basedVSB-spline-based_and_NN-basedVSdirectNN} compares the proposed NN-based plug-in classifier with trajectory-based classifiers, including FNNs, RNNs, TCNs, and Transformers. These methods take the whole observed path as input and predict class labels directly, without exploiting the underlying SDE structure. As shown in Figure~\ref{fig:NN-basedVSB-spline-based_and_NN-basedVSdirectNN}, exploiting the diffusion structure substantially improves classification performance. A detailed discussion is provided in Appendix~\ref{app:trajectory-comparison}, including a theoretical comparison with FNN-based trajectory classifiers, whose convergence rate is of order $N^{-\frac{1}{2}\frac{(1+\alpha)\beta}{(1+\alpha)\beta+(M+1)d}}$ up to logarithmic factors, and numerical comparisons highlighting the effective use of data.

\begin{comment}
\textbf{NN-based plug-in classifier vs. trajectory-based classifier.}
Figure~\ref{fig:NN-basedVSB-spline-based_and_NN-basedVSdirectNN} compares the proposed NN-based plug-in classifier (blue circular markers) with trajectory-based classifiers (red square markers) that take the entire observed path as input, including feedforward neural networks (referred to as FNNs in what follows; see, e.g., \citet{bos2022}), RNNs, TCNs, and Transformers, which ignore the underlying SDE structure and predict class labels directly from the observed trajectories.

As shown in Figure~\ref{fig:NN-basedVSB-spline-based_and_NN-basedVSdirectNN}, exploiting the diffusion structure substantially improves classification performance. Direct trajectory-based classifiers treat each path as a high-dimensional input vector, whose effective dimension grows with the number of observation times, which can make learning statistically less efficient. A detailed discussion is provided in Appendix~\ref{app:trajectory-comparison}, including a theoretical comparison with FNN-based trajectory classifiers and numerical comparisons on the effective use of data.

\end{comment}

\begin{figure}[ht]
%  \vskip 0.2in
  \begin{center}
%     \begin{subfigure}{0.49\linewidth}
%        \includegraphics[width=\linewidth]{figures-final/nnvsbspline.png}
%        \caption{NN-based vs. B-spline-based plug-in classifiers}
%        \label{fig:NN-basedVSB-spline-based}
%    \end{subfigure}

    \begin{subfigure}{0.24\linewidth}
        \includegraphics[width=\linewidth]{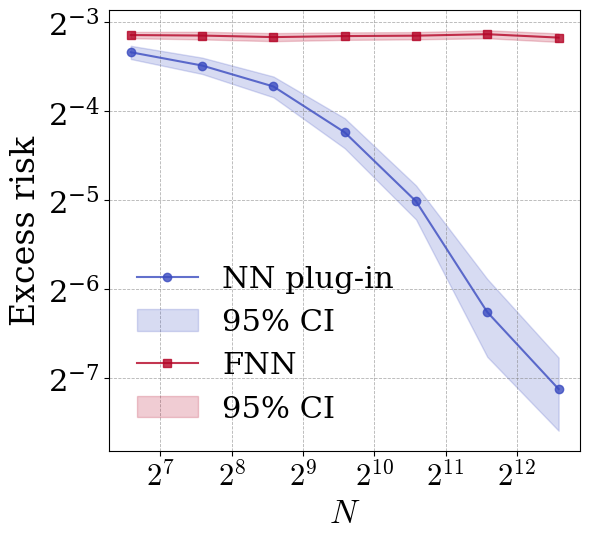}
%       \caption{NN-based plug-in classifier vs. direct NN classification}
 %       \label{fig:NN-basedVSdirectNN}
    \end{subfigure}
    \hfill
     \begin{subfigure}{0.24\linewidth}
        \includegraphics[width=\linewidth]{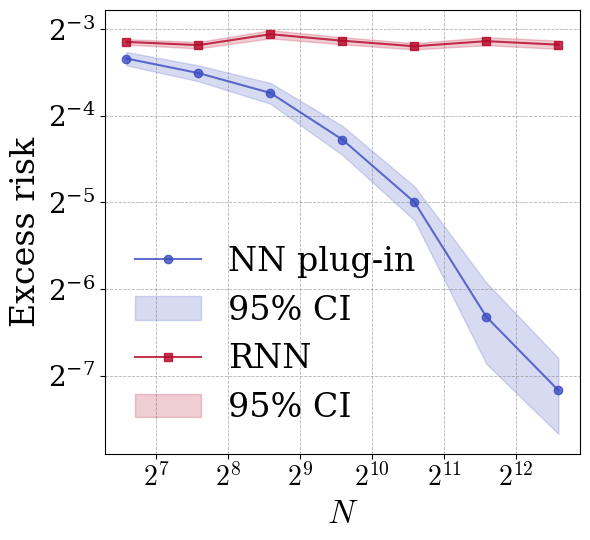}
%        \caption{NN-based vs. B-spline-based plug-in classifiers}
%        \label{fig:NN-basedVSB-spline-based}
    \end{subfigure}
     \begin{subfigure}{0.24\linewidth}
        \includegraphics[width=\linewidth]{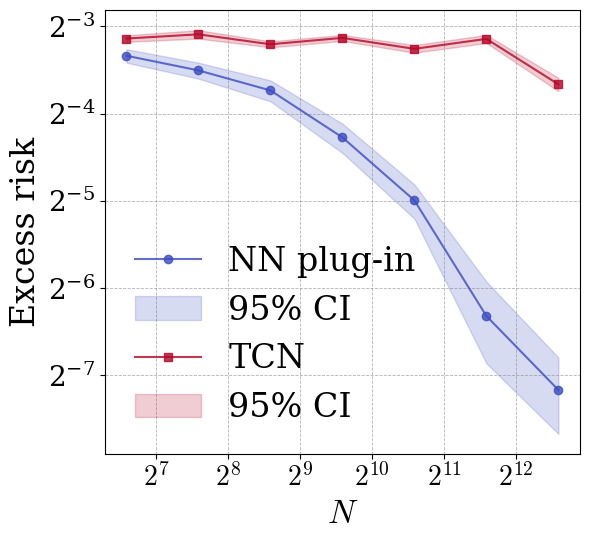}
%        \caption{NN-based vs. B-spline-based plug-in classifiers}
%        \label{fig:NN-basedVSB-spline-based}
    \end{subfigure}
    \hfill
         \begin{subfigure}{0.24\linewidth}
        \includegraphics[width=\linewidth]{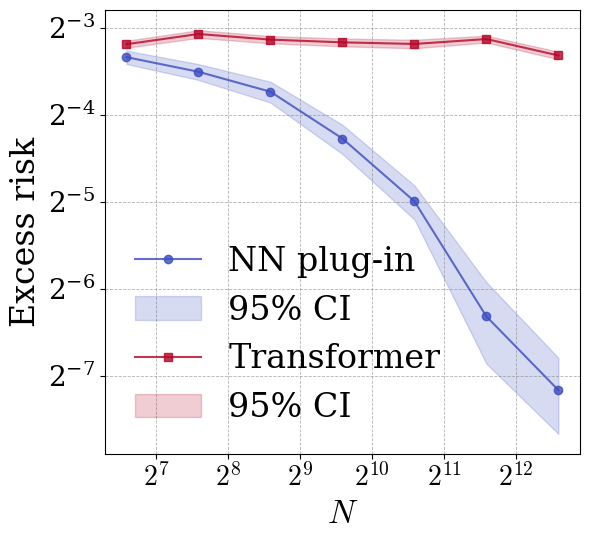}
%        \caption{NN-based vs. B-spline-based plug-in classifiers}
%        \label{fig:NN-basedVSB-spline-based}
    \end{subfigure}
    \caption{Comparison of the NN-based plug-in classifier with trajectory-based classifiers. From left to right: FNN, RNN, TCN, and Transformer.}
%\caption{Comparison between the NN-based plug-in classifier and trajectory-based classifiers: feedforward neural network (top left), RNN (top right), TCN (bottom left), and Transformer (bottom right).}
\label{fig:NN-basedVSB-spline-based_and_NN-basedVSdirectNN}
  \end{center}
\end{figure}

\begin{figure}[ht]
    \centering
    % First row
    \begin{subfigure}{0.24\linewidth}
        \includegraphics[width=\linewidth]{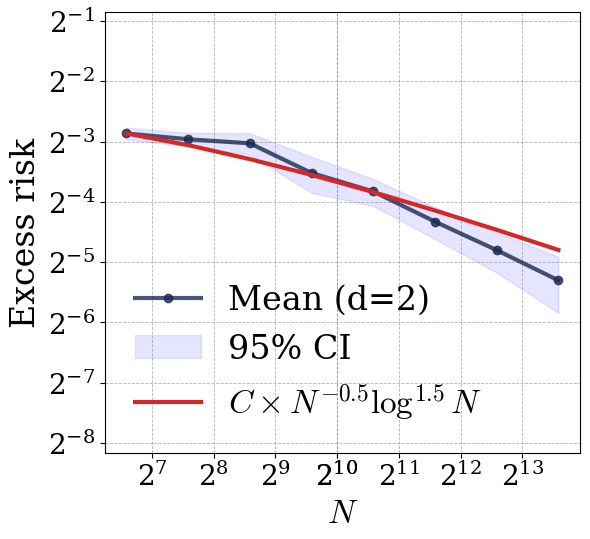}
%        \caption{}
        \label{fig:sub1}
    \end{subfigure}
    \hfill
    \begin{subfigure}{0.24\linewidth}
        \includegraphics[width=\linewidth]{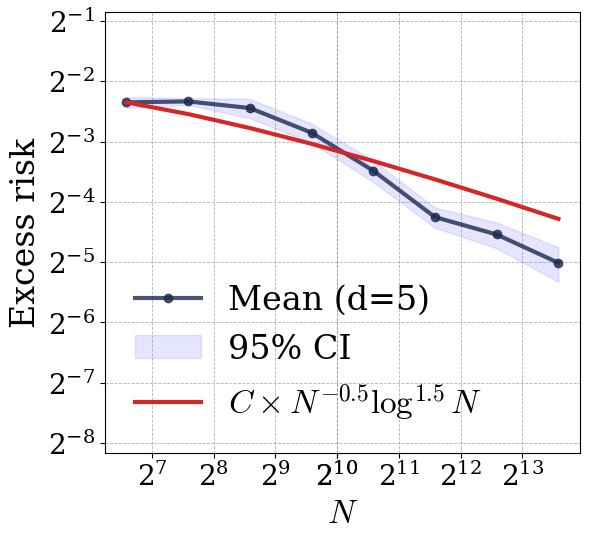}
%       \caption{}
        \label{fig:sub2}
    \end{subfigure}
    \begin{subfigure}{0.24\linewidth}
        \includegraphics[width=\linewidth]{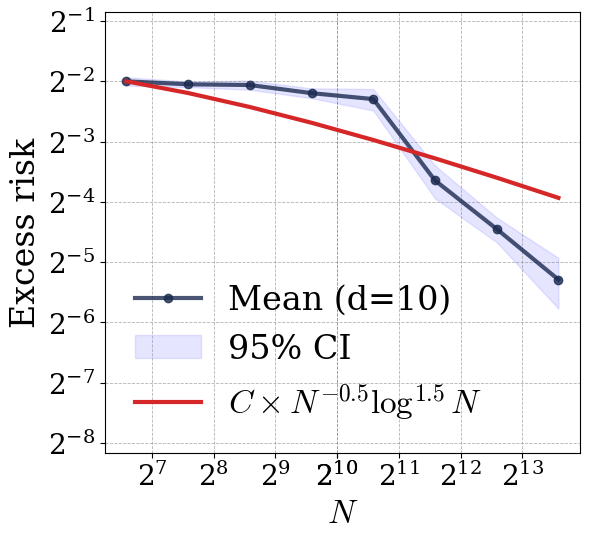}
%        \caption{}
        \label{fig:sub4}
    \end{subfigure}
    \hfill
    \begin{subfigure}{0.24\linewidth} 
        \includegraphics[width=\linewidth]{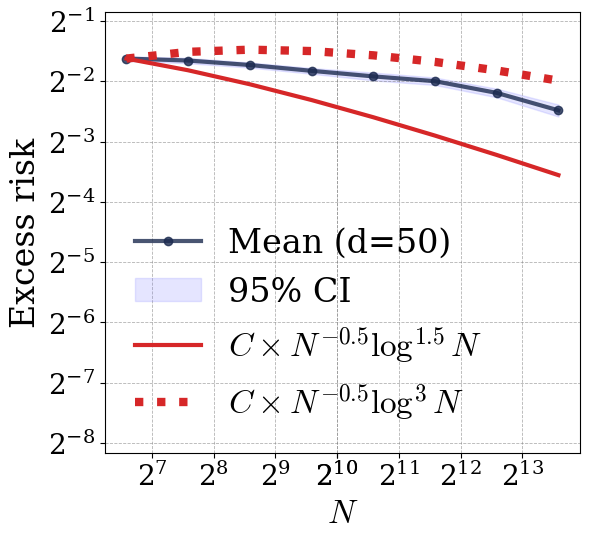}
   %     \caption{}
        \label{fig:sub5}
    \end{subfigure}
    \caption{Convergence rates of the excess classification risk of the NN-based classifier for $d = 2$, $5$, $10$, and $50$, from left to right.}
 %   \caption{Convergence rates of the excess classification risk of NN-based classifier for $ d = 2 $ (top left), $ d = 5 $ (top right), $ d = 10 $ (bottom left), and $ d = 50 $ (bottom right). }
    \label{fig:convergence}
\end{figure}

%Direct NN classification paper: 

\textbf{High-dimensional convergence rate of the NN-based plug-in classifier.}
%We next investigate the empirical convergence rate of the excess classification risk of the NN-based plug-in classifier in high-dimensional settings, with dimensions $d \in \{2,5,10,50\}$. The true drift functions $b_k$, $k = 1,2,3$, defined in \eqref{eq:numerical-def-b_k}, are sufficiently smooth. As a result, the theoretical convergence rate predicted by Theorem~\ref{thm:main-NN} is
%\[
%N^{-\frac{1}{2}+\varepsilon} \log^{\frac{3}{2}-\varepsilon} N,
%\]
%for any arbitrarily small $\varepsilon > 0$.
We next investigate the empirical convergence rate of the excess classification risk of the NN-based plug-in classifier in high-dimensional settings, with dimensions $d \in \{2,5,10,50\}$. The true drift functions $b_k$, $k = 1,2,3$, defined in \eqref{eq:numerical-def-b_k}, are sufficiently smooth. As a result, the theoretical convergence rate predicted by Theorem~\ref{thm:main-NN} is
$
N^{-\frac{1}{2}+\varepsilon} \log^{\frac{3}{2}-\varepsilon} N,
$
for any arbitrarily small $\varepsilon > 0$.

%Figure~\ref{fig:convergence} indicates that the empirical convergence rate of the NN-based plug-in classifier is consistent with a rate of order $N^{-1/2}(\log N)^a$ for $a\in \{\frac{3}{2},3\}$ across all considered dimensions. This observation is in agreement with the theoretical bound derived in Theorem~\ref{thm:main-NN}, suggesting that the rate is essentially sharp up to logarithmic factors. Given the narrow hyperparameter set for the architure of the neural network, it proves that even with a small set of hyperparameter, NN based classifier can acheive the good proformance in this task, which is numerically efficace. 

Figure~\ref{fig:convergence} indicates that the empirical convergence rate of the NN-based plug-in classifier is consistent with a rate of order $N^{-1/2}(\log N)^a$, for $a \in \{\tfrac{3}{2}, 3\}$, across all considered dimensions. This observation is in agreement with the theoretical bound derived in Theorem~\ref{thm:main-NN}, suggesting that the rate is essentially sharp up to logarithmic factors. Moreover, the experiments are conducted using a very limited set of neural network hyperparameters, highlighting its numerical efficiency.

%shows the empricial convergence rate of the log-log classification excess risk of our NN based plug-in classfier. The x axis is log N where N is the number of sample for every t

%[TO DO]

%[TO DO]

\subsection{Example 2: Example in \citet{Denis2024}}\label{sec:cos-example}

%\paragraph{Experimental setup}

In this section, we revisit the example originally introduced in \citet{Denis2024}, with drift functions defined by
\begin{equation}\label{eq:numerical-def-b_k-cos}
b_k(x) = \alpha_k \theta \Big(\tfrac{1}{4} + \tfrac{3}{4}\cos^2 x\Big),
\quad k = 1,2,3,
\end{equation}
$X_0=0$, and a diffusion coefficient defined by $\sigma(x)=0.1+\tfrac{0.9}{\sqrt{1+x^2}}$. Here, $(\alpha_1, \alpha_2, \alpha_3)=(\frac{1}{\theta}, 1, -1)$ and $\theta$ is a parameter controlling the separation between classes. Following \citet{Denis2024}, different values of $\theta$ are considered in order to assess the performance of the classification methods under varying levels of difficulty. Details on data generation and implementation are provided in Appendix~\ref{app:Experimental-3.2}. %The NN-based and B-spline–based classifiers follow the same configurations as in Section~\ref{sec:double-layer-example} and \citep[Section~3 and 5.2]{Denis2024}, respectively.

Table~\ref{table:cos-example} shows that, for this example, our NN-based plug-in classifier achieves comparable performance in terms of classification risk to the B-spline–based classifier proposed in \citet{Denis2024}. The first two columns of Table~\ref{table:cos-example} are reproduced from \citet[Table 1]{Denis2024} for reference. 
\begin{table}[th]
  \begin{minipage}{0.5\linewidth}
      \begin{sc}
        \begin{tabular}{lcccr}
          \toprule
          $\theta$  & Bayes        &  B-spline       & NN \\
          &error & based & based\\
          \midrule
          0.5    & 0.49 & 0.53 &  0.505\\
          1.5 & 0.36 & 0.39 & 0.392 \\
          2.5    & 0.22 & 0.34 & 0.239 \\
          4    & 0.11 & 0.12 &   0.117    \\
          \bottomrule
        \end{tabular}
      \end{sc}
  \end{minipage}
    \begin{minipage}{0.5\linewidth}
      \begin{sc}
        \begin{tabular}{lcccr}
          \toprule
          $\theta$  & Bayes        &  B-spline       & NN \\
          &error & based & based\\
          \midrule
          0.5   & 0.49 & 0.50 & 0.494 \\
          1.5 & 0.36 & 0.37 & 0.376 \\
          2.5    & 0.22 & 0.22 & 0.227 \\
          4    & 0.11 & 0.10 &   0.107      \\
          \bottomrule
        \end{tabular}
      \end{sc}
  \end{minipage}
  \smallskip
\caption{Comparison of classification risk between B-spline--based and NN-based plug-in classifiers for the example introduced in \citet{Denis2024}, with $N=100$ (left) and $N=1000$ (right).}
  \label{table:cos-example}
\end{table}

\begin{comment}

\begin{table}[th]

  \begin{center}
    \begin{small}
    
      \begin{sc}
        \begin{tabular}{lcccr}
          \toprule
          $\theta$  & Bayes error       &  B-spline based      & NN-based \\
          \midrule
          0.5    & 0.49 & 0.53 &  0.505\\
          1.5 & 0.36 & 0.39 & 0.392 \\
          2.5    & 0.22 & 0.34 & 0.239 \\
          4    & 0.11 & 0.12 &   0.117    \\
          \bottomrule
        \end{tabular}
      \end{sc}
\smallskip
      \begin{sc}
        \begin{tabular}{lcccr}
          \toprule
          $\theta$  & Bayes error       &  B-spline based      & NN-based \\
          \midrule
          0.5   & 0.49 & 0.50 & 0.494 \\
          1.5 & 0.36 & 0.37 & 0.376 \\
          2.5    & 0.22 & 0.22 & 0.227 \\
          4    & 0.11 & 0.10 &   0.107      \\
          \bottomrule
        \end{tabular}
      \end{sc}
    \end{small}
  \end{center}
\caption{Comparison of classification risk between B-spline--based and NN-based plug-in classifiers for the example introduced in \citet{Denis2024}, with $N=100$ (top) and $N=1000$ (bottom).}
  \label{table:cos-example}
  \vskip -0.1in
\end{table}
\end{comment}

\section{Conclusion}\label{sec:conclusion}

The classification task is intrinsically more challenging than drift estimation task, since the excess classification risk \eqref{eq:def-excess classification risk} is always bounded by~$1$. In particular, \eqref{eq:inequality-in-main-thm} shows that when the drift estimator is not sufficiently accurate, the resulting classifier behaves nearly at random. This difficulty is further amplified in high-dimensional settings, as also observed in \citet[Section~7]{denis2025empirical}.

\textbf{Limitation.} 
A careful analysis of the proof of Theorem~\ref{thm:main-decomposition} shows that the constant in \eqref{eq:inequality-in-main-thm} grows with the dimension~$d$. This dependence arises from repeated applications of the Burkholder--Davis--Gundy (BDG) inequality and Gronwall's lemma in Lemma~\ref{lem:properties-X} of \citet[Lemma~7.4 and Proposition~7.6]{pages2018numerical}, whose associated constants grow with the dimension~$d$. In the setting of Example~1, the drift functions satisfy the compositional structural assumptions described in Theorem~\ref{thm:main-NN}. Consequently, the dimension affects the constants in the bound, but not the exponent of the convergence rate. The bound therefore remains meaningful in multidimensional settings, although larger sample sizes may be needed to observe the predicted decrease in excess risk as \(d\) increases. This phenomenon is also illustrated in Figure~\ref{fig:convergence}, which exhibits a clear inflection point: as the dimension~$d$ increases, a larger sample size~$N$ is required before the excess risk begins to decrease significantly. Moreover, the requirement $\varepsilon > 0$ in Theorem~\ref{thm:main-NN} is a technical artifact of the proof. Indeed, the constant $C_{K,\Lambda,\mathfrak{C},\widetilde{C},\mathcal{K},\varepsilon}$
diverges as $\varepsilon \to 0$. See the proof of Theorem~\ref{thm:main-NN} for details.

\begin{comment}

\textbf{Limitation.} A careful analysis of the proof of Theorem~\ref{thm:main-decomposition} shows that the constant in \eqref{eq:inequality-in-main-thm} grows with the dimension~$d$. This dependence arises from repeated applications of the Burkholder–Davis–Gundy (BDG) inequality and Gronwall’s lemma in Lemma~\ref{lem:properties-X} of \citet[Lemma~7.4 and Proposition~7.6]{pages2018numerical}, whose associated constants grow with the dimension $d$. Consequently, higher-dimensional problems require increasingly accurate drift estimation to achieve reliable classification. This phenomenon is illustrated in Figure~\ref{fig:convergence}, which exhibits a clear inflection point: as the dimension $d$ increases, a larger sample size $N$ is required before the classification error begins to decrease significantly. Moreover, the requirement $\varepsilon > 0$ in Theorem~\ref{thm:main-NN} is a technical artifact of the proof. Indeed, the constant $C_{K,\Lambda,\mathfrak{C},\widetilde{C},\mathcal{K},\varepsilon}$
diverges as $\varepsilon \to 0$. See the proof of Theorem~\ref{thm:main-NN} for details.
\end{comment}

%Despite these intrinsic difficulties, the proposed neural network–based plug-in classifier achieves strong performance both theoretically and empirically. By explicitly exploiting the structure of the underlying diffusion process, our approach remains effective even in high-dimensional setting and outperforms the trajectory based classifier, as confirmed by our theoretical guarantees and numerical experiments.

Despite these intrinsic difficulties, the proposed neural network--based plug-in classifier achieves strong theoretical and empirical performance. By explicitly exploiting the structure of the underlying diffusion process, our approach remains effective in multidimensional settings and outperforms trajectory-based classifiers,  as confirmed by our theoretical guarantees and numerical experiments.

%Gives this diffucility, our NN-based plug in classifier acheive good performance both theotically and numerically in this paper. 

%it is worth to note that classification problem is more difficulte than simple estimation problem since the excess risk is lower than 1, hence a no good estimation will leed directly random classification. This can be reflected in the Figure \ref{fig:convergence}inflection point, when the estimation error is lower than some level, the classification error becomes significant. 

%explain the change in the figure 

\section*{Acknowledgments}
Yuzhen Zhao acknowledges support from the Research Chair DIALog under the aegis of the Risk Foundation, a joint initiative by Université Paris-Dauphine PSL and CNP Assurances. Yating Liu acknowledges the financial support from the CNRS through the MITI interdisciplinary programs.

\bibliographystyle{plainnat}
\bibliography{neurips_2026}

%%%%%%%%%%%%%%%%%%%%%%%%%%%%%%%%%%%%%%%%%%%%%%%%%%%%%%%%%%%%%%%%%%%%%%%%%%%%%%%
%%%%%%%%%%%%%%%%%%%%%%%%%%%%%%%%%%%%%%%%%%%%%%%%%%%%%%%%%%%%%%%%%%%%%%%%%%%%%%%
% APPENDIX
%%%%%%%%%%%%%%%%%%%%%%%%%%%%%%%%%%%%%%%%%%%%%%%%%%%%%%%%%%%%%%%%%%%%%%%%%%%%%%%
%%%%%%%%%%%%%%%%%%%%%%%%%%%%%%%%%%%%%%%%%%%%%%%%%%%%%%%%%%%%%%%%%%%%%%%%%%%%%%%
%\newpage
\appendix

\section{Detailed proof}\label{sec:app-a}
\begin{proof}[Proof of Proposition \ref{prop:bayes-classifier}]
First, recall that original probability measure $\mathbb{P}$ admits the mixture representation
$
\mathbb{P} = \sum_{k=1}^K \mathfrak{p}_k \, \mathbb{P}_k,
$ 
where $\mathbb{P}_k(\cdot) \coloneqq \mathbb{P}(\,\cdot \mid Y = k)$ and $\mathfrak{p}_k = \mathbb{P}(Y = k)$. Assumption~\ref{assump:Lipschitz} guarantees the existence and strong uniqueness of the solution to~\eqref{eq:sde} under each measure $\mathbb{P}_k$, for all $1 \le k \le K$. 

The proof of Proposition~\ref{prop:bayes-classifier} is divided into two steps. 
In the first step, starting from $\mathbb{P}_k$, we construct a reference probability measure $\tilde{\mathbb{P}}$ on $(\Omega, \mathcal{F}, (\mathcal{F}_t)_{t\in[0,T]})$, and a Brownian motion $(\tilde{B}_t)_{t\in[0,T]}$ under $\tilde{\mathbb{P}}$ such that the process $(\tilde{X}_t)_{t\in[0,T]}$ defined by
\begin{equation}\label{eq:driftless-sde}
\d \tilde{X}_t = \sigma(\tilde{X}_t)\,\d \tilde{B}_t,\quad\tilde{X}_0=X_0    
\end{equation}
solves the original SDE \eqref{eq:sde} under ${\mathbb{P}}$.  This construction relies on Girsanov's theorem.  In the second step, we prove~\eqref{eq:pi-equality}, which follows the proof of~\citet[Proposition~1]{Denis2020} and is reproduced here for  reader’s convenience.

\noindent {\sc Step 1.} Fix a label $k \in \mathcal{Y}$.
Define the process
\begin{align}\label{eq:appendix-martingale}
    Z_t^{[k]} %&\coloneqq \exp \left[\sum_{i=1}^d \int_{0}^t \big[\sigma^{-1}b_k(X_s)\big]_i \d \tilde{B}_s^i -\frac{1}{2}\int_{0}^t |\sigma^{-1}b_k(X_s)|^2\d s\right],\nonumber\\
    & = \exp \left[\sum_{i=1}^d \int_{0}^t - \big[\sigma^{-1}(X_s)\,b_k(X_s)\big]_i \d B_s^i -\frac{1}{2}\int_{0}^t |\sigma^{-1}(X_s)\,b_k(X_s)|^2\d s\right], \; t\in[0,T],
\end{align}
where $\sigma^{-1}(x)$ denotes the inverse of $\sigma(x)$, and $\big[\sigma^{-1}(X_s)\,b_k(X_s)\big]_i$ is the $i$-th coordinate of $\sigma^{-1}(X_s)\,b_k(X_s)$, and $B^i$ denotes the $i$-th component of the
Brownian motion $B$. Assumption \ref{assump:Novikov} ensures that $(Z^{[k]}_t)_{t\in[0,T]}$ defined by \eqref{eq:appendix-martingale} is a true martingale (see, e.g., \citet[Corollary 5.13]{karatzas1991brownian}).
 
We now define the probability measure $\tilde\PP$ on $\calF_T$ by \[\frac{\d \tilde{\PP}}{\d \PP_k} =Z_T^{[k]}, \quad \text{ which means, }\quad  
\tilde\PP(A) \coloneqq \EE_k\Big[\mathbbm{1}_A Z_T^{[k]}\Big], \:A\in\calF_T,
\]
where $\EE_k$ is the expectation with respect to $\PP_k$. Hence, by Girsanov's Theorem (see, e.g., \citet[Theorem 5.1]{karatzas1991brownian}),  the process $\tilde{B}=(\tilde{B}_t)_{t\in[0,T]}$  defined by 
\begin{equation}\label{eq:appendix-BM}
\tilde{B}_t  \coloneqq B_t+\int_{0}^t \sigma^{-1}(X_s)\,b_k(X_s)\d s, \:t\in[0,T] 
\end{equation} 
is a Brownian motion under $\tilde\PP$. 

It remains to verify that the process
$X = (X_t)_{t\in[0,T]}$, which solves~\eqref{eq:sde} under $\PP_k$,
satisfies the driftless SDE~\eqref{eq:driftless-sde} under $\tilde{\PP}$.
Indeed,
\begin{align}
\d X_t = b_k(X_t)\,\d t + \sigma(X_t)\,\d B_t= b_k(X_t)\,\d t + \sigma(X_t)\,\big(\d \tilde{B}_t - \sigma^{-1}(X_t)\,b_k(X_t)\d t\big)=\sigma(X_t)\d \tilde{B}_t.
\end{align}
This concludes the first step, since Assumption~\ref{assump:Lipschitz}
ensures existence and strong uniqueness of solutions to both
\eqref{eq:sde} and~\eqref{eq:driftless-sde}.

\noindent {\sc Step 2.} Since $(Z^{[k]}_t)_{t\in[0,T]}$ defined in \eqref{eq:appendix-martingale} is a true martingale (see {\sc Step 1}), the probability measures $\tilde{\PP}$ and $\PP_k$ are mutually absolutely continuous (see, e.g., \citet[Corollary 5.2 and the subsequent remark]{karatzas1991brownian}).  We therefore define, for every $t\in[0,T]$, 
\begin{align}\label{eq:def-phi-t-k}
&\Phi_t^{[k]}\coloneqq (Z_{t}^{[k]})^{-1}  =  \exp \left[\sum_{i=1}^d \int_{0}^t  \big[\sigma^{-1}(X_s)\,b_k(X_s)\big]_i \d B_s^i +\frac{1}{2}\int_{0}^t |\sigma^{-1}(X_s)\,b_k(X_s)|^2\d s\right]\nonumber\\
%&= \exp \left[\sum_{i=1}^d \int_{0}^t \big[\sigma^{-1}b_k(X_s)\big]_i \d B_s^i +\frac{1}{2}\int_{0}^t |\sigma^{-1}b_k(X_s)|^2\d s\right],\nonumber\\
    & =\exp \left[\int_{0}^t \big[\sigma^{-1}(X_s)\,b_k(X_s)\big]^{\top} \d B_s +\frac{1}{2}\int_{0}^t |\sigma^{-1}(X_s)\,b_k(X_s)|^2\d s\right],\nonumber\\
    &=\exp \Bigg[\int_{0}^t \big[\sigma^{-1}(X_s)\,b_k(X_s)\big]^{\top} \big(\sigma^{-1}(X_s)\d X_s - \sigma^{-1}(X_s)\,b_k(X_s)\d s\big) \nonumber\\
    &\qquad +\frac{1}{2}\int_{0}^t |\sigma^{-1}(X_s)\,b_k(X_s)|^2\d s\Bigg],\nonumber\\
    &=\exp \left[\int_{0}^t b_k(X_s)^{\top} (\sigma\sigma^{\top})^{-1}(X_s)\d X_s  -\frac{1}{2}\int_{0}^t |\sigma^{-1}(X_s)\,b_k(X_s)|^2\d s\right].
\end{align}
Consequently, 
\begin{equation}\label{eq:def-dpk-dptilde}
    \frac{\d \PP_k}{\d \tilde{\PP}}=\Phi_T^{[k]}
\end{equation} (see, e.g., \citet[Section 5.5, Consequence (a) after Theorem 5.22]{legall2016brownian}). It follows that
\[
\d \PP = \sum_{k=1}^K \mathfrak{p}_k\d \PP_{k} =  \sum_{k=1}^K \mathfrak{p}_k \Phi_T^{[k]}\d \tilde{\PP}, 
\]
Moreover, since each measure $\PP_k$, $k \in \mathcal{Y}$, is mutually absolutely continuous with respect to $\tilde{\PP}$, the same holds for $\PP$ and $\tilde{\PP}$.
Therefore, the Radon–Nikodym theorem implies that
\begin{equation}\label{eq:app-density-psi}
\frac{\d \PP_k}{\d \PP}=\frac{\d \PP_k}{\d \tilde\PP}\cdot \frac{\d \tilde\PP}{\d \PP}=\frac{ \Phi_T^{[k]}}{ \sum_{j=1}^K \mathfrak{p}_j \Phi_T^{[j]}}\eqqcolon \Psi^{[k]}.
\end{equation}

%\begin{equation}\label{eq:app-density-psi}
%\frac{\d \PP_k}{\d \PP}=\frac{ \Phi_T^{[k]}\d \tilde{\PP}}{ \sum_{j=1}^K \mathfrak{p}_j \Phi_T^{[j]}\d \tilde{\PP}}=\frac{ \Phi_T^{[k]}}{ \sum_{j=1}^K \mathfrak{p}_j \Phi_T^{[j]}}\eqqcolon \Psi^{[k]}.
%\end{equation}

Let $\calF^X$ denote the $\sigma$-algebra generated by
$X = (X_t)_{t\in[0,T]}$, the solution of~\eqref{eq:sde}.
Let $h:\{1, ..., K\}\rightarrow \RR$ be a  bounded measurable function and let $Z$ be a bounded $\mathcal{F}^X$-measurable random variable. Then
\begin{align}
    \EE [h(Y)Z]= \EE\big[\,h(Y) \,\EE[Z|Y]\,\big] =\sum_{k=1}^K h(k)\mathfrak{p}_k \EE[Z|Y=k]=\sum_{k=1}^K h(k)\mathfrak{p}_k \EE_k[Z]
\end{align}
where $\EE_k[\cdot]$ denotes expectation with respect to $\PP_k$. 
Using~\eqref{eq:app-density-psi}, we obtain
\begin{align}
   \EE [h(Y)Z]= \sum_{k=1}^K h(k)\mathfrak{p}_k \EE\big[Z\Psi^{[k]}\big]=\EE \left[\sum_{k=1}^K h(k)\mathfrak{p}_k \Psi^{[k]} Z\right]. 
\end{align}
By the definition of conditional expectation
(see, e.g., \citet[Theorem and Definition~11.3]{legall2022measure}),
this implies
\begin{equation}\label{eq:app-cond-expectation}
    \EE[h(Y)|X]=\sum_{k=1}^K h(k)\,\mathfrak{p}_k \Psi^{[k]}, \quad \PP-\text{a.s.,}
\end{equation}
Finally, choosing $h(y) = \mathbbm{1}_{\{k\}}(y)$ for every $k \in \mathcal{Y}$
concludes the proof.
\end{proof}

Recall that \( \mathfrak{C} \) denotes a generic positive constant, depending only on the model parameters \( d, T, b_1,\ldots,b_K, \sigma \) and \( \|X_0\|_4 \), whose value may change from line to line.
 The proof of  Theorem \ref{thm:main-decomposition} relies on the following lemma, whose proof can be found in \citet[Lemma 7.4 and Proposition 7.6]{pages2018numerical}. 
\begin{lemma}\label{lem:properties-X}
Let $X=(X_t)_{t\in [0,T]}$ be the unique solution to \eqref{eq:sde}. Under Assumption \ref{assump:Lipschitz}, the coefficient functions $b_k, \,k\in\calY$ and $\sigma$ have linear growth, i.e., for any $x\in\RD$, %it holds 
\[\max (|b_1(x)|, \dots,|b_K(x)| , |\sigma(x)|)\leq \mathfrak{C}(1+|x|).  \]
Moreover, 
\[  \sup_{t\in[0,T]}|| X_t\Vert_4\leq \mathfrak{C}, \;  \textrm{ and for every } \,s,t\in[0,T], \quad \Vert X_t -X_s \Vert_2\leq \Vert X_t -X_s \Vert_4\leq \mathfrak{C} |t-s|^{\frac{1}{2}}. \]
\end{lemma}

\begin{proof}[Proof of Theorem \ref{thm:main-decomposition}]
We start by recalling that, by Proposition~2 in \citet{Denis2020},
\begin{align}
\calR(\widehat{g})-\mathcal{R}(g^*)=\EE \Big[ \sum_{i=1}^K\sum_{\substack{k=1\\ k\neq i\,}}^K \big( \pi_i^*(X)-\pi_k^*(X)\big)\mathbbm{1}_{\{\widehat g(X)=k\}}\mathbbm{1}_{\{g^*(X)=i\}}\Big].
\end{align}
Consequently,
\begin{align}
\calR(\widehat{g})-\mathcal{R}(g^*)&\leq \EE \Big[ \sum_{i=1}^K\sum_{k=1, k\neq i}^K \big( \pi_i^*(X)-\widehat \pi_i(X) + \widehat \pi_k(X)-\pi_k^*(X)\big)\mathbbm{1}_{\{\widehat g(X)=k\}}\mathbbm{1}_{\{g^*(X)=i\}}\Big]\nonumber\\
&\leq \EE \Big[ 2 \max_{i\in\mathcal{Y}}| \pi_i^*(X)-\widehat \pi_i(X)|  \sum_{i=1}^K\sum_{k=1, k\neq i}^K \mathbbm{1}_{\{\widehat g(X)=k\}}\mathbbm{1}_{\{g^*(X)=i\}} \Big]\nonumber\\
&=\EE \Big[ 2 \max_{i\in\mathcal{Y}}| \pi_i^*(X)-\widehat \pi_i(X)|  \mathbbm{1}_{\{\widehat g(X)\neq g^*(X)\}}\Big]\leq 2\EE \Big[  \max_{i\in\mathcal{Y}}| \pi_i^*(X)-\widehat \pi_i(X)| \Big]\nonumber\\
&\leq 2 \sum_{k=1}^{K}\,\EE \big[ \big|\pi_k^*(X)-\widehat \pi_k(X) \big|\big],\nonumber
\end{align}
where the first inequality holds because on the event 
\( \{\widehat g(X)=k\} \), the definition of \( \widehat g \) in
\eqref{eq:def-pi-hat-g-hat} implies \( \widehat \pi_k(X) \ge \widehat \pi_i(X) \) for all $i\neq k$.

We now define
\[\bar\pi_k(X) = \phi_k\big( \bar{F}(X) \big),\quad  k\in \calY,\]
where \( \phi_k \) denotes the softmax function defined in
\eqref{eq:def-phi}, \( \bar F = (\bar F_1,\ldots,\bar F_K) \), and each component
\( \bar F_k \) is given by \eqref{eq:F-bar}.
For each fixed label \( k \in \mathcal Y \), %we obtain
\begin{align}\label{eq:ineq-pi-and-F}
    \EE \big[ \big|\pi_k^*(X)-\widehat \pi_k(X) \big|\big]&\leq \EE \big[ \big|\pi_k^*(X)-\bar \pi_k(X) \big|\big]+\EE \big[ \big|\bar\pi_k(X)-\widehat \pi_k(X) \big|\big]\nonumber\\
    &\leq \EE \big[ \big|F^*(X)-\bar F(X) \big|\big]+\EE \big[ \big|\bar F(X)-\widehat F(X) \big|\big]\nonumber\\
    &\leq \sum_{k=1}^K \left(\EE \big[ \big|F_k^*(X)-\bar F_k(X) \big|\big]+\EE \big[ \big|\bar F_k(X)-\widehat F_k(X) \big|\big]\right)
\end{align}
where the second inequality uses the fact that the softmax functions
\( \phi_k : \mathbb R^K \to \mathbb R_+, \:k\in \calY\) are 1-Lipschitz continuous with respect to the Euclidean norm \( |\cdot| \).

We next derive upper bounds for
$\EE \big[ \big|F_k^*(X)-\bar F_k(X) \big|\big]$ and $\EE \big[ \big|\bar F_k(X)-\widehat F_k(X) \big|\big]$, $k\in\calY$. 

{\sc Step 1. Upper bound for $\EE \big[ \big|F_k^*(X)-\bar F_k(X) \big|\big]$.}

The functional $\bar F_k(X)$ defined by \eqref{eq:F-bar} can be rewritten as 
\begin{align}\label{eq:F-bar-rewrite}
\bar{F}_k(X)
%\coloneqq
%&\sum_{m=0}^{M-1} b_k(X_{t_m})^{\top} \big(\sigma\sigma^{\top}\big)^{-1}(X_{t_m})\,(X_{t_{m+1}}\!\!-X_{t_m}) -\frac{\Delta}{2}\sum_{m=0}^{M-1} \big\lvert \sigma^{-1} (X_{t_m})\,b_k(X_{t_m}) \big\rvert^2\nonumber\\
=& \int_{0}^{T}b_k(X_{\eta(s)})^{\top} \big(\sigma\sigma^{\top}\big)^{-1}(X_{\eta(s)})\d X_s - \frac{1}{2}\int_{0}^{T} \big\lvert \sigma^{-1} (X_{\eta(s)})\,b_k(X_{\eta(s)}) \big\rvert^2 \d s,
\end{align}
where, for every $s\in[0, T]$, 
\begin{equation}\label{eq:def-eta(s)}
    \eta (s)\coloneqq t_m, \text{ if } t_m\leq s < t_{m+1}, \text{ and }\eta(T)=T.
\end{equation}%$ if $$ and $$. 

Let \[a(x)\coloneqq\sigma\sigma^{\top}(x), \:h_k(x)\coloneqq b_k(x)^{\top} a(x)^{-1}\text{ and }\ell_k(x)=|\sigma^{-1}(x)\,b_k(x)|^2=b_k(x)^{\top} a(x)^{-1} b_k(x).\] Then 
\begin{align}
F_k^*(X)-\bar F_k(X) = \int_0^T \big( h_k(X_s)-h_k(X_{\eta(s)})\big)\d X_s - \frac{1}{2}\int_0^T \big( \ell_k(X_s)-\ell_k(X_{\eta(s)})\big)\d s. 
\end{align}
Hence, 
\begin{align}
&\EE \big[ \big|F_k^*(X)-\bar F_k(X) \big|\big] \nonumber\\
&\quad\leq \EE \left[ \left|\int_0^T \big( h_k(X_s)-h_k(X_{\eta(s)})\big)\d X_s\right|\right]  + \frac{1}{2} \EE \left[\int_0^T \big| \ell_k(X_s)-\ell_k(X_{\eta(s)})\big|\d s\right]\eqqcolon I_1 + \frac{1}{2} I_2\nonumber
\end{align}

{\sc Step 1-a. Upper bound of $I_1$}

Let $x,y \in \RD$, we have 
\begin{align}\label{eq:lip-h}
    |h_k(x)&-h_k(y)|\nonumber\\
    &=|b_k(x)^{\top} a(x)^{-1}-b_k(y)^{\top} a(y)^{-1} |\nonumber\\
    &=|b_k(x)^{\top} a(x)^{-1}-b_k(x)^{\top} a(y)^{-1}+ b_k(x)^{\top} a(y)^{-1}-b_k(y)^{\top} a(y)^{-1}| \nonumber\\
    &=|b_k(x)^{\top} \big(a(x)^{-1}- a(y)^{-1}\big)+ \big(b_k(x)-b_k(y)\big)^{\top}a(y)^{-1} |\nonumber\\
    &\leq|b_k(x)|\cdot | a(x)^{-1}- a(y)^{-1} |_{\text{op}} + |b_k(x)-b_k(y)| \cdot | a(y)^{-1}|_{\text{op}} \nonumber\\
    &\leq |b_k(x)|\cdot L_{a^{-1}}|x-y| + L_b|x-y|\cdot |a(y)^{-1}|_{\text{op}} + 
    \leq (L_b\Lambda + L_{a^{-1}}|b_k(x)|)|x-y|
\end{align}

%\begin{align}
%    |h_k(x)-h_k(y)|&=|b_k(x)^{\top} a(x)^{-1}-b_k(y)^{\top} a(y)^{-1} |\nonumber\\
%    &=|b_k(x)^{\top} a(x)^{-1}-b_k(y)^{\top} a(x)^{-1}+ b_k(y)^{\top} a(x)^{-1}-b_k(y)^{\top} a(y)^{-1}| \nonumber\\
%    &=|\big(b_k(x)-b_k(y)\big)^{\top}a(x)^{-1} + b_k(y)^{\top} \big(a(x)^{-1}- a(y)^{-1}\big)|\nonumber\\
%    &\leq |b_k(x)-b_k(y)| \cdot \Vert a(x)^{-1}\Vert_{\text{op}}+ |b_k(y)|\cdot \Vert  a(x)^{-1}- a(y)^{-1} \Vert_{\text{op}} \nonumber\\
%    &\leq L_b|x-y|\cdot \Vert a(x)^{-1}\Vert_{\text{op}} + |b_k(y)|\cdot L_{a^{-1}}|x-y|
%    \leq (L_b\Lambda + L_{a^{-1}}|b_k(y)|)|x-y|
%\end{align}
%where $\Vert \cdot \Vert_{\text{op}}$ denotes the operator norm of a matrix, that is, ... 

Hence, 
\begin{align}
    I_1 &= \EE \left[ \left|\int_0^T \big( h_k(X_s)-h_k(X_{\eta(s)})\big)\d X_s\right|\right] \nonumber\\
    &= \EE \left[ \left|\int_0^T \big( h_k(X_s)-h_k(X_{\eta(s)})\big)\big(b_Y(X_s)\d s+\sigma (X_s)\d B_s\big)\right|\right]  \nonumber\\
    &\leq \EE \left[\left|\int_0^T \big( h_k(X_s)-h_k(X_{\eta(s)})\big)b_Y(X_s)\d s\right|\right] + \EE \left[\left|\int_0^T \big( h_k(X_s)-h_k(X_{\eta(s)})\big) \sigma (X_s)\d B_s\right|\right]\nonumber\\
    &\eqqcolon J_1+ J_2.
\end{align}
Now for $J_1$, we have
\begin{align}
J_1 &\leq \EE \left[ \int_0^T  \big| h_k(X_s)-h_k(X_{\eta(s)})\big| \big|b_Y(X_s)\big| \d s\right] = \int_0^T  \EE \left[ \big| h_k(X_s)-h_k(X_{\eta(s)})\big| \big|b_Y(X_s)\big|\right ]\d s \nonumber\\
&\leq \int_0^T\EE \Big[ \big| (L_b\Lambda + L_{a^{-1}}|b_k(X_s)|)|X_s-X_{\eta(s)}|\big| \cdot \big|b_Y(X_s)\big|\Big ]\d s \nonumber\\
&\leq \int_0^T \left \Vert X_s-X_{\eta(s)} \right\Vert_2 \mathfrak{C}\Big(1+  \big(\EE[ |X_s |^4]\big)^{\frac{1}{2}}\Big) \d s\nonumber\\
&\leq \mathfrak{C} \sqrt{\Delta},
\end{align}
where the first equality follows from Fubini’s theorem, the second inequality follows from \eqref{eq:lip-h}, and the last two inequalities from Hölder’s inequality and Lemma~\ref{lem:properties-X}.

For $J_2$, we have  %using Burkhölder–Davis–Gundy (BDG) Inequality (see e.g. \citep[Inequality (7.53)]{pages2018numerical}), we obtain
\begin{align}
(J_2)^2&\leq \EE \left[\left|\int_0^T \big( h_k(X_s)-h_k(X_{\eta(s)})\big) \sigma (X_s)\d B_s\right|^2\right]\nonumber\\
&= \EE \left[\int_0^T \big| h_k(X_s)-h_k(X_{\eta(s)})\sigma (X_s)\big|^2\d s\right]\nonumber\\
&\leq \int_0^T \EE \left[\big| h_k(X_s)-h_k(X_{\eta(s)})\big|^2|\sigma(X_s)|^2_{\text{op}}\right]\d s \nonumber\\
&\leq  \int_0^T \EE \left[\big| h_k(X_s)-h_k(X_{\eta(s)})\big|^2|\sigma(X_s)|^2\right]\d s \nonumber\\
&\leq \Lambda\int_0^T \left\Vert h_k(X_s)-h_k(X_{\eta(s)})\right\Vert_2^2 \d s\leq \mathfrak{C} \Delta,
\end{align}
%\rr error $\sigma \sigma^{\top}$\black
where the first inequality follows from Jensen’s inequality, the equality from Itô’s isometry, the second inequality from Fubini’s theorem and the definition of the operator norm, the next inequality from the fact that $|A|_{\text{op}}\leq |A|$ for any matrix, and the final bound from Lemma~\ref{lem:properties-X}.

Therefore, $I_1\leq \mathfrak{C}\sqrt{\Delta}$. 

{\sc Step 1-b. Upper bound of $I_2$}

Let $x,y \in \RD$, we have 
\begin{align}
   &|\ell_k(x)-\ell_k(y)|\nonumber\\
   &=|b_k(x)^{\top} a(x)^{-1} b_k(x)-b_k(y)^{\top} a(y)^{-1} b_k(y)|\nonumber\\
   &\leq |b_k(x)^{\top} a(x)^{-1} b_k(x)-b_k(y)^{\top} a(x)^{-1} b_k(x) |+|b_k(y)^{\top} a(x)^{-1} b_k(x) -b_k(y)^{\top} a(y)^{-1} b_k(x) |\nonumber\\
   &\quad \:+|b_k(y)^{\top} a(y)^{-1} b_k(x)  - b_k(y)^{\top} a(y)^{-1} b_k(y)|\nonumber\\
   &\leq |(b_k(x)-b_k(y))^{\top}a(x)^{-1} b_k(x)|+|b_k(y)^{\top} \big(a(x)^{-1}-a(y)^{-1}\big) b_k(x) |\nonumber\\
   &\qquad +b_k(y)^{\top} a(y)^{-1} \big(b_k(x)  -  b_k(y)\big)|\nonumber\\
   & \leq \Lambda L_b|x-y||b_k(x)|+ L_{a^{-1}}|b_k(y)||b_k(x)||x-y|+\Lambda L_b|b_k(y)||x-y| \nonumber\\
   &=\Big( \Lambda L_b\big(|b_k(x)|+|b_k(y)|\big)+L_{a^{-1}}|b_k(y)||b_k(x)|\Big)|x-y|. 
\end{align}
It follows that, by Hölder's inequality and Lemma~\ref{lem:properties-X},
\begin{align}
    I_2&=\int_0^T \EE \left[\big| \ell_k(X_s)-\ell_k(X_{\eta(s)})\big|\right]\d s\nonumber\\
    &\leq \int_0^T \EE \left[\Big( \Lambda L_b\big(|b_k(X_s)|+|b_k(X_{\eta(s)})|\big)+L_{a^{-1}}|b_k(X_s)||b_k(X_{\eta(s)})|\Big)|X_s-X_{\eta(s)}|\right]\d s\nonumber\\
    &\leq \int_0^T\left\Vert \Lambda L_b\big(|b_k(X_s)|+|b_k(X_{\eta(s)})|\big)+L_{a^{-1}}|b_k(X_s)||b_k(X_{\eta(s)})|\right\Vert_2\left \Vert X_s-X_{\eta(s)}\right\Vert_2 \d s\nonumber\\
    &\leq \mathfrak{C}\Delta^{\frac{1}{2}} 
\end{align}
Consequently, $\EE \big[ \big|F_k^*(X)-\bar F_k(X) \big|\big]\leq \mathfrak{C}\sqrt{\Delta}$.

%\[I_2=\EE \left[\int_0^T \big| g_k(X_s)-g_k(X_{\eta(s)})\big|\d s\right]\]

{\sc Step 2. Upper bound of $\EE \big[ \big|\bar F_k(X)-\widehat F_k(X) \big|\big]$.}

Using $\eta(s)$ from \eqref{eq:def-eta(s)} and $a=\sigma\sigma^{\top}$, we can write
\begin{align}\label{eq:F-bar-rewrite}
&\bar{F}_k(X)
= \int_{0}^{T}b_k(X_{\eta(s)})^{\top} \big(\sigma\sigma^{\top}\big)^{-1}(X_{\eta(s)})\d X_s - \frac{1}{2}\int_{0}^{T} \big\lvert \sigma^{-1} (X_{\eta(s)})\,b_k(X_{\eta(s)}) \big\rvert^2 \d s,\nonumber\\
&\widehat{F}_k(X)=\int_{0}^{T} \widehat{b}_k(X_{\eta(s)})^{\top} \big(\sigma\sigma^{\top}\big)^{-1}(X_{\eta(s)})\d X_s - \frac{1}{2}\int_{0}^{T} \big\lvert \sigma^{-1} (X_{\eta(s)})\,\widehat{b}_k(X_{\eta(s)}) \big\rvert^2 \d s.
\end{align}
Therefore, 
\begin{align}\label{eq:part-I-and-II}
    \EE &\left[ |\bar F_k(X)-\widehat F_k(X)|\right]\nonumber\\
    &\leq  \EE \left[\left|\int_{0}^T \big( b_k(X_{\eta(s)})-\widehat b_k(X_{\eta(s)})\big)^{\top}a^{-1}(X_{\eta(s)})\d X_s\right|\right]\nonumber\\
     &\quad +\frac{1}{2}\EE \left[\left|\int_0^T \left[\big\lvert \sigma^{-1} (X_{\eta(s)})\,b_k(X_{\eta(s)}) \big\rvert^2-\big\lvert \sigma^{-1} (X_{\eta(s)})\,\widehat{b}_k(X_{\eta(s)}) \big\rvert^2\right]\d s\right|\right]\nonumber\\
     &\eqqcolon \text{(I) + $\frac{1}{2}$(II)}.
\end{align}
For part (I), we have
\begin{align}\label{eq:I}
 \text{(I)} & =   \EE \left[\left|\int_{0}^T \big( b_k(X_{\eta(s)})-\widehat b_k(X_{\eta(s)})\big)^{\top}a^{-1}(X_{\eta(s)})\Big(b_Y(X_s) \d s + \sigma(X_s) \d B_s \Big)\right|\right]\nonumber\\
 & \leq \EE \left[\left|\int_{0}^T \big( b_k(X_{\eta(s)})-\widehat b_k(X_{\eta(s)})\big)^{\top}a^{-1}(X_{\eta(s)})b_Y(X_s) \d s \right|\right]\nonumber\\
 &\quad +\EE \left[\left|\int_{0}^T \big( b_k(X_{\eta(s)})-\widehat b_k(X_{\eta(s)})\big)^{\top}a^{-1}(X_{\eta(s)})\sigma(X_s) \d B_s \right|\right]\eqqcolon \text{(I.a) + (I.b)}.
\end{align}
Next, for (I.a), using the Cauchy–Schwarz inequality, the definition of the operator norm and Lemma \ref{lem:properties-X},  we obtain 
\begin{align}\label{eq:I-a}
    \text{(I.a)}&\leq \EE \left[\int_{0}^T \left|\big( b_k(X_{\eta(s)})-\widehat b_k(X_{\eta(s)})\big)^{\top}a^{-1}(X_{\eta(s)})b_Y(X_s) \right|\d s \right]\nonumber\\
    &\leq \Lambda\EE \left[\int_{0}^T \left|\big( b_k(X_{\eta(s)})-\widehat b_k(X_{\eta(s)})\big)\right| \left|b_Y(X_s) \right|\d s \right]\nonumber\\
    &\leq \Lambda\EE \left[\left(\int_{0}^T \left|\big( b_k(X_{\eta(s)})-\widehat b_k(X_{\eta(s)})\big)\right|^2\d s\right)^{\frac{1}{2}}\left(\int_{0}^T\left|b_Y(X_s) \right|^2\d s\right)^{\frac{1}{2}} \right]\nonumber\\
    &\leq \Lambda \left(\EE \left[\int_{0}^T \left|\big( b_k(X_{\eta(s)})-\widehat b_k(X_{\eta(s)})\big)\right|^2\d s\right]\right)^{\frac{1}{2}}\left(\EE \left[\int_{0}^T\left|b_Y(X_s) \right|^2\d s \right]\right)^{\frac{1}{2}}\nonumber\\
    &\leq \Lambda \mathfrak{C}\left(\EE \left[\Delta\sum_{m=0}^{M-1}\left|\big( b_k(X_{t_m})-\widehat b_k(X_{t_m})\big)\right|^2\right]\right)^{\frac{1}{2}}=\Lambda \mathfrak{C}\,  \calE(\widehat b_k, b_k)^{\frac{1}{2}}.
\end{align}
%where the first inequality follows from the Cauchy–Schwarz inequality and the definition of the operator norm, the second and third inequality follows from the Cauchy–Schwarz inequality 
%\begin{align}
%    \EE \big[ \big|\bar F_k(X)-\widehat F_k(X) \big|\big]\leq 
%\end{align}
Similarly, for part (I.b), using Jensen's inequality and It\^o isometry, we obtain
\begin{align}\label{eq:I-b}
\text{(I.b)}&\leq \left\{    \EE \left[\left|\int_{0}^T \big( b_k(X_{\eta(s)})-\widehat b_k(X_{\eta(s)})\big)^{\top}a^{-1}(X_{\eta(s)})\sigma(X_s) \d B_s \right|^2\right]\right\}^{\frac{1}{2}}\nonumber\\
&=\left\{    \EE \left[\int_{0}^T \left|\big( b_k(X_{\eta(s)})-\widehat b_k(X_{\eta(s)})\big)^{\top}a^{-1}(X_{\eta(s)})\sigma(X_s)\right|^2 \d s \right]\right\}^{\frac{1}{2}}\nonumber\\
&\leq \left\{    \int_{0}^T \EE \left[\left| b_k(X_{\eta(s)})-\widehat b_k(X_{\eta(s)})\right|^2\left|a^{-1}(X_{\eta(s)})\sigma(X_s)\right|^2 \right]\d s \right\}^{\frac{1}{2}}\nonumber\\
&\leq \Lambda^2 \left\{    \int_{0}^T \EE \left[\left| b_k(X_{\eta(s)})-\widehat b_k(X_{\eta(s)}\right|^2\right]\d s \right\}^{\frac{1}{2}} = \Lambda^2  T \calE(\widehat b_k, b_k)^{\frac{1}{2}}. 
\end{align}
For part (II) of \eqref{eq:part-I-and-II}, using the inequality \[\forall\,u, v\in\RD, \quad \big||u|^2-|v|^2\big|=\big|(u-v)^{\top}(u+v)\big|\leq |u-v||u+v| ,\] 
we obtain
\begin{align}\label{eq:II}
\text{(II)} & =    \EE \left[\left|\int_0^T \left[\big\lvert \sigma^{-1} (X_{\eta(s)})\,b_k(X_{\eta(s)}) \big\rvert^2-\big\lvert \sigma^{-1} (X_{\eta(s)})\,\widehat{b}_k(X_{\eta(s)}) \big\rvert^2\right]\d s\right|\right]\nonumber\\
&\leq  \EE \left[\int_{0}^{T} \left| \sigma^{-1} (X_{\eta(s)})\big(b_k(X_{\eta(s)}) -\widehat b_k(X_{\eta(s)}) \big)\right|  \left|\sigma^{-1} (X_{\eta(s)})\big(b_k(X_{\eta(s)}) +\widehat b_k(X_{\eta(s)}) \big)\right|  \d s\right]\nonumber\\
&\leq \left\{ \EE \left[\int_{0}^{T} \left| \sigma^{-1} (X_{\eta(s)})\big(b_k(X_{\eta(s)}) -\widehat b_k(X_{\eta(s)}) \big)\right|^2 \d s \right]\right\}^{\frac{1}{2}}\nonumber\\
&\qquad\times\left\{ \EE \left[\int_{0}^{T} \left| \sigma^{-1} (X_{\eta(s)})\big(b_k(X_{\eta(s)}) +\widehat b_k(X_{\eta(s)}) \big)\right|^2 \d s \right]\right\}^{\frac{1}{2}}\nonumber\\
&\eqqcolon \mathrm{(II.a)}^{\frac{1}{2}}\times \mathrm{(II.b)}^{\frac{1}{2}}.
\end{align}
For (II.a), we have 
\begin{align}\label{eq:II-a}
 \mathrm{(II.a)} & =    \EE \left[\int_{0}^{T} \left| \big(b_k(X_{\eta(s)}) -\widehat b_k(X_{\eta(s)}) \big)^{\top}a^{-1} (X_{\eta(s)})\big(b_k(X_{\eta(s)}) -\widehat b_k(X_{\eta(s)}) \big)\right| \d s \right]\nonumber\\
 &\leq \Lambda \EE \left[\int_{0}^{T} \left| b_k(X_{\eta(s)}) -\widehat b_k(X_{\eta(s)}) \right|^2 \d s\right]   = \Lambda T \calE(\widehat b_k, b_k).
\end{align}
%\rr [attention error ici ... $X$ vient d'un melange de class et pas juste de classe k ici ] \black
For (II.b), we have 
\begin{align}\label{eq:II-b}
     \mathrm{(II.b)} & \leq \Lambda \EE \left[\int_{0}^{T} \left| b_k(X_{\eta(s)}) +\widehat b_k(X_{\eta(s)}) \right|^2 \d s\right] \leq 2 \Lambda T (\mathfrak{C}+\widehat b_{\max}^{\,2})
\end{align}
Combining the above bounds yields
\begin{align}
    \EE \left[ |\bar F_k(X)-\widehat F_k(X)|\right]\leq C_{\Lambda, \mathfrak{C}, \widehat b_{\max}}\calE(\widehat b_k, b_k)^{\frac{1}{2}}
\end{align}
Finally, we conclude that $\calR(\widehat{g})-\mathcal{R}(g^*)\leq K^2 C_{\Lambda, \mathfrak{C}, \widehat b_{\max}}\big(\sqrt{\Delta} + \max_{k\in\calY}\calE(\widehat b_k, b_k)^{\frac{1}{2}}\big)$, which completes the proof.
\end{proof}

Note that the upper bound in \eqref{eq:inequality-in-main-thm} involves the
estimation error \(\mathcal{E}(\widehat b_k, b_k)\), which is defined without
conditioning on any specific class label (see \eqref{eq:def-estimation-error-all}). In practice, however, each drift
function \(b_k\) is estimated only from training trajectories belonging to
class \(k\).
The following Lemma~\ref{lem:relation-E-with-Ek} therefore establishes a
connection between the global estimation error
\(\mathcal{E}(\widehat b_k, b_k)\) and the class-conditional estimation error
\(\mathcal{E}_k(\widehat b_k, b_k)\) defined in
\eqref{eq:def-estimation-error-class-k}.

\begin{lemma}\label{lem:relation-E-with-Ek} Let $j,k\in\calY$ with $j\neq k$. 
Under the assumptions of Theorem~\ref{thm:main-NN}, for any $\varepsilon\in(0,\frac{1}{2}]$, there exists a constant
$C_{\Lambda,\mathfrak{C},\varepsilon}>0$ such that
\[
\mathcal{E}_j(\widehat b_k,b_k)\leq C_{\Lambda, \mathfrak{C}, \varepsilon}\mathcal{E}_k(\widehat b_k,b_k)^{1-\varepsilon}. 
\]
Moreover, $C_{\Lambda, \mathfrak{C}, \varepsilon}\rightarrow+\infty$ when $\varepsilon\rightarrow 0$. 
\end{lemma}

\begin{proof}[Proof of Lemma \ref{lem:relation-E-with-Ek}]
Under the assumptions of Theorem~\ref{thm:main-NN}, each $b_k\in \mathcal{G}(q,\mathbf{d}, \mathbf{t},\boldsymbol{\beta}, K)$. In particular, $b_k$ is defined on a compact set. Hence,  there exists a constant $b_{\max}>0$ such that, for every $k\in\calY$, $\Vert b_k\Vert_{\sup}\leq b_{\max}$.

Let $\theta_k(\cdot)\coloneqq \big(\sigma^{-1}(\cdot)b_k(\cdot)\big)^{\top}$ and $\xi_k(\cdot)\coloneqq \theta_k(\cdot)\theta_k^{\top}(\cdot)=|\sigma^{-1}(\cdot)b_k(\cdot)|^2$. Then the process $\Phi_t^{[k]}$ defined in \eqref{eq:def-phi-t-k} can be written as
\[
\Phi_t^{[k]} = \exp \left [  \int_{0}^{t}\theta_k(X_s)\d B_s + \frac{1}{2}\int_{0}^t\xi_k(X_s)\d s\right]. 
\]
Moreover, by \eqref{eq:def-dpk-dptilde}, we have 
\begin{align}\label{eq:upper-bound-change-density}
    \frac{\d \PP_j}{\d \PP_k}= \frac{\Phi_T^{[j]}}{\Phi_T^{[k]}}&= \exp \left [  \int_{0}^{T}\big(\theta_j(X_s)-\theta_k(X_s)\big)\d B_s + \frac{1}{2}\int_{0}^T\big(\xi_j(X_s)-\xi_k(X_s)\big)\d s\right]\nonumber\\
    &\leq C_{\Lambda, \mathfrak{C}}  \exp \left [  \int_{0}^{T}\big(\theta_j(X_s)-\theta_k(X_s)\big)\d B_s\right],
\end{align}
where the inequality follows from the boundedness of $b_\ell$ for $\ell\in\mathcal Y$
and Assumption~\ref{assump:main-theorem}.

Now for every $j,k\in\calY, \,j\neq k$, we define
\begin{equation}\label{eq:def-M-jk}
 M_{t}^{j,k}\coloneqq \int_{0}^{t}\big(\theta_j(X_s)-\theta_k(X_s)\big)\d B_s, \:t\in[0,T].    
\end{equation}
Then, for any fixed $\varepsilon\in(0,\frac{1}{2}]$, we have 
\begin{align}\label{eq:change-measure-j-to-k}
\mathcal{E}_j&(\widehat b_k,b_k)
\coloneqq
\EE_j \left[
\frac{1}{M}\sum_{m=0}^{M-1}
\Big(\widehat b_k(X_{t_m})-b_k(X_{t_m})\Big)^2
\right]\nonumber\\
&= \EE_k \left[
\left(\frac{1}{M}\sum_{m=0}^{M-1}
\Big(\widehat b_k(X_{t_m})-b_k(X_{t_m})\Big)^2 \right)\frac{\d \PP_j}{\d \PP_k}
\right]\nonumber\\
&\leq C_{\Lambda, \mathfrak{C}}\EE_k \left[
\left(\frac{1}{M}\sum_{m=0}^{M-1}
\Big(\widehat b_k(X_{t_m})-b_k(X_{t_m})\Big)^2 \right)\exp\big(M_T^{j,k}\big)
\right]\nonumber\\
%&\leq C_{\Lambda, \mathfrak{C}}\exp(a)\mathcal{E}_k(\widehat b_k,b_k)+C_{\Lambda, \mathfrak{C}}(b_{\max}+F)^2\EE \left[\exp\big(M_T^{j,k}\big)\mathbbm{1}_{\{M_T^{j,k}>a\}}\right],
&\leq C_{\Lambda, \mathfrak{C}} \left\{\EE_k \left[
\left(\frac{1}{M}\sum_{m=0}^{M-1}
\Big(\widehat b_k(X_{t_m})-b_k(X_{t_m})\Big)^2 \right)^{1+\frac{\varepsilon}{1-\varepsilon}}\right]\right\}^{1-\varepsilon} \left\{ \EE_k \left[\exp\left(\frac{1}{\varepsilon}M_T^{j,k}\right)\right] \right\}^{\varepsilon},
\end{align}
where the first inequality follows from \eqref{eq:upper-bound-change-density} and the definition of $M_t^{j,k}$ in \eqref{eq:def-M-jk}, and the second one from H\"older inequality. 
It follows that 
\begin{align}
    \EE_k \left[
\left(\frac{1}{M}\sum_{m=0}^{M-1}
\Big(\widehat b_k(X_{t_m})-b_k(X_{t_m})\Big)^2 \right)^{1+\frac{\varepsilon}{1-\varepsilon}}\right]\leq (b_{\max}+F)^2\mathcal{E}_k(\widehat b_k,b_k),
\end{align}
since $\widehat b_k$ and $b_k$ are respectively bounded by $F$ and $b_{\max}$ and $\frac{\varepsilon}{1-\varepsilon}\leq 1$ for $\varepsilon\in(0,\frac{1}{2}]$. Moreover, 
\begin{align}
    \EE_k \left[\exp\left(\frac{1}{\varepsilon}M_T^{j,k}\right)\right] =\EE_k \left[\exp\left(\frac{1}{\varepsilon}M_T^{j,k}-\frac{1}{2\varepsilon^2}\langle M^{j,k}\rangle_T\right)\exp\left(\frac{1}{2\varepsilon^2}\langle M^{j,k}\rangle_T\right)\right]. 
\end{align}
Since the functions $b_k, k\in\calY$ are bounded, Assumption \ref{assump:main-theorem} implies that  $\xi_k(\cdot)$ is uniformly bounded by $b_{\max}^2 \Lambda$. Hence 
\[\exp\left(\frac{1}{2\varepsilon^2}\langle M^{j,k}\rangle_T\right)\leq \exp\left(\frac{Tb_{\max}^2 \Lambda}{\varepsilon^2}\right)<+\infty.\]
Therefore, \citet[Proposition 5.12 and Corollary 5.13]{karatzas1991brownian} implies that 
\[\EE_k \left[\exp\left(\frac{1}{\varepsilon}M_T^{j,k}-\frac{1}{2\varepsilon^2}\langle M^{j,k}\rangle_T\right)\right]=1. \]
Finally, gathering the above inequalities together gives 
$\mathcal{E}_j(\widehat b_k,b_k)\leq C_{\Lambda, \mathfrak{C}, \varepsilon}\mathcal{E}_k(\widehat b_k,b_k)^{1-\varepsilon}$.\end{proof}

\begin{proof}[Proof of Theorem \ref{thm:main-NN}]

By Lemma~\ref{lem:relation-E-with-Ek} and the definitions of
$\mathcal{E}(\widehat b_k,b_k)$ and $\mathcal{E}_j(\widehat b_k,b_k)$ in
\eqref{eq:def-estimation-error-all} and
\eqref{eq:def-estimation-error-class-k}, we have
\[\mathcal{E}(\widehat b_k, b_k)=\sum_{j\in\calY}\mathfrak{p_j}\mathcal{E}_j(\widehat b_k, b_k)\leq C_{\Lambda, \mathfrak{C}, \varepsilon}\mathcal{E}_k(\widehat b_k,b_k)^{1-\varepsilon}+\mathfrak{p}_k \mathcal{E}_k(\widehat b_k,b_k).\]
Moreover, by Corollary~2.3 of \citet{zhao2025drift}, the estimation error of 
$b_k$ satisfies  
\[
\mathcal{E}_k(\widehat b_k,b_k)
\le
C\,\phi_N\,\log^3 N .
\]
Combining this bound with Theorem~\ref{thm:main-decomposition} yields the
desired result. This concludes the proof of Theorem~\ref{thm:main-NN}.
\end{proof}

\section{Experimental Details}\label{app:Experimental-Details}

\subsection{Example in Section \ref{sec:double-layer-example}}\label{app:Experimental-3.1}

\paragraph{Data Generation.}
Sample paths of \eqref{eq:sde} are simulated using the Euler--Maruyama discretization scheme over the time interval $[0,1]$, discretized with $M = 100$ steps and time step $\Delta = 1/M = 0.01$. Experiments are performed for dimensions $d \in \{1,2,5,10,50\}$.

Recall that, in this example, $K=3$. For each dimension $d$, and for each label $k\in\calY=\{1,2,3\}$, we generate $N_k\in \{2^5,2^6,\dots,2^{12}\}$ independent trajectories as the training set, denoted by  %grouped by 
\[
\mathcal{D}^{[k]}_{N_k} = \bigl\{\, \bar{X}^{[k],(n)}_{t_0:t_M} \,\bigr\}_{1\leq n\leq N_k}. 
\]
%where $N_k $ for all $k \in \mathcal{Y}$. %We set $K = 3$.

For evaluation, we independently generate 1000 trajectories for each label $k\in\calY$, denoted by 
\[
\mathcal{D}_{\mathrm{test}} := \bigcup_{k \in \mathcal{Y}} \mathcal{D}^{[k]}_{\mathrm{test}} \quad \text{ with }\quad
\mathcal{D}^{[k]}_{\mathrm{test}} := \bigl\{\, \tilde{X}^{[k],(n')}_{t_0:t_M}  \,\bigr\}_{1\leq n' \leq 1000},\quad k \in \mathcal{Y}.
\]

\paragraph{NN-based plug-in classifier.}  \:%\\ %\label{app:nn-plugin}

{\sc Architecture.} For the neural network estimator for $b_k$,  we use the sparse network architecture $\mathbf{p}= (d, 16, 32, 32, 16, 1)$ with ReLU activation function and the sparsity ratio $s_{\text{ratio}} = 0.75$ inspired by \citet{zhao2025drift}. 
% The number of nonzero parameters $s$ is controlled as a proportion $s_{\mathrm{ratio}}$ of the total number of parameters, with $s_{\text{ratio}} = 0.75$. And the sparsity constraint $s$ used in the estimator class $\mathcal{F}(L, \mathbf{p}, s, F)$ \eqref{eq:sparseNN} is computed as $$s = \left\lfloor s_{\text{ratio}} \cdot \sum_{i=0}^{L} (p_i + 1)\cdot p_{i+1} \right\rfloor,$$ where $L$ is the number of hidden layers and $(p_0,\dots, p_{L+1})$ is the layer width vector.

{\sc Training Procedure.}
Training is performed using the Adam optimizer with a learning rate of $10^{-3}$ and batch size of 256.  For each experimental configuration, specified by a choice of the dimension $d$, the sample size $N$, and a random seed, we randomly reserve 50\% of the training samples as a validation set for early stopping. Early stopping is triggered if the validation error does not improve for 20 consecutive epochs, with a maximum of 200 training epochs. Once early stopping is activated, we retrain the model on the full training dataset for a number of epochs equal to twice the epoch at which early stopping occurred.

%For each experimental setting, we randomly keep 50\% validation samples to conduct early stopping, which is triggered after 20 consecutive epochs without improvement of validation errors on validation set. Training is performed for up to 200 epochs. Once early stopping is triggered, we retrain the model using the entire training dataset  for a total number of epochs equal to twice the number of epochs observed at the early-stopping point.

During training, the training samples are shuffled and processed in mini-batches. After each optimizer update, we enforce the sparsity by keeping only the top 75\% parameters (by absolute value) and clip all parameters to the interval $[-1, 1]$. 

Random seeds for NumPy and PyTorch are fixed in each repetition to ensure reproducibility.

\paragraph{B-spline–based plug-in classifier.}%\label{app:bspline}
For the B-spline estimator for $b_k$, we adopt the same procedure (including tuning and implementation details) as in \citet{Denis2024}.
% and its accompanying code.

%\paragraph{Direct neural network classification.}%\label{app:nn-direct}

\paragraph{Trajectory-based classifier by feedforward neural network.}
A direct pathwise classification approach is considered, where the full discretely observed path
$$ \bar{X}^{[k],(n)}_{t_0:t_M}
 :=
 \big(
\bar X^{[k],(n)}_{t_0},
 \dots,
\bar X^{[k],(n)}_{t_M}
 \big)$$
is fed into a neural network, which directly outputs a prediction of the class label $k\in\calY$ associated with the trajectory. This approach reduces the problem to a standard multiclass classification task and does not exploit the underlying SDE structure.

A fully connected neural network with two hidden layers and ReLU activation functions is employed.
For each experimental setting, the training data are randomly split, with 50\% of the samples allocated to a validation set for hyperparameter tuning. Hyperparameters are optimized with Optuna \citep{akiba2019optuna}, and the configuration that maximizes the validation accuracy is selected. The hyperparameter search space is defined as follows: %The hyperparameter search space consists of:
\begin{itemize}
    \item learning rate $\in \{10^{-4},\, 3\times10^{-4},\, 10^{-3},\, 3\times10^{-3}\}$,
    \item weight decay $\in\{0,\, 10^{-5},\, 
    10^{-4},\, 10^{-3}\}$,
    \item hidden layers size $\in\{(16,16),\, (32,32),\, (64,64),\, (128,128),\, (256,128)\}$,
    \item batch size $\in\{64,\, 128,\, 256\}$.
\end{itemize}
Early stopping based on the validation accuracy is applied, with a patience of 50 epochs and a maximum budget of 200 epochs. The model achieving the highest validation accuracy during training is retained. 

Random seeds for NumPy and PyTorch are fixed in each repetition to ensure reproducibility.

% The direct NN classifier exhibits an error rate above 60\%. When the total training size satisfies $N \le 3075$, its performance is essentially indistinguishable from random guessing, whose error rate is $66.67\%$ in the three-class setting.

\paragraph{Trajectory-based classifier by RNN.}
A vanilla RNN with $\tanh$ activation is employed for direct pathwise classification.
The final RNN hidden state is used as the trajectory representation and fed into a linear classification classifier.

The same validation split, Optuna-based tuning procedure, early-stopping rule, and random-seed setting are used as for the feedforward neural network classifier. 
The hyperparameter search space is defined as follows:

\begin{itemize}

    \item learning rate $\in \{10^{-4},\, 3\times10^{-4},\, 10^{-3},\, 3\times10^{-3}\}$,

    \item weight decay $\in\{0,\, 10^{-5},\, 10^{-4},\, 10^{-3}\}$,

    \item recurrent hidden dimension $\in\{32,\,64,\,128\}$,

    \item number of recurrent layers $\in\{1,\,2\}$,

    \item dropout rate $\in\{0,\,0.1,\,0.2\}$,

    \item batch size $\in\{64,\,128,\,256\}$.

\end{itemize}

\paragraph{Trajectory-based classifier by TCN.}
A temporal convolutional network is used for direct pathwise classification.
The TCN consists of residual one-dimensional convolutional blocks with ReLU activation and dropout.
The temporal features are averaged over time and fed into a linear classification classifier.

The same validation split, Optuna-based tuning procedure, early-stopping rule, and random-seed setting are used as for the feedforward neural network classifier. 
The hyperparameter search space is defined as follows:

\begin{itemize}

    \item learning rate $\in \{10^{-4},\, 3\times10^{-4},\, 10^{-3},\, 3\times10^{-3}\}$,

    \item weight decay $\in\{0,\, 10^{-5},\, 10^{-4},\, 10^{-3}\}$,

    \item channel sizes $\in\{(16,32,32),\, (32,64,64),\, (32,64,128)\}$,

    \item kernel size $\in\{3,\,5\}$,

    \item dropout rate $\in\{0,\,0.1,\,0.2\}$,

    \item batch size $\in\{64,\,128,\,256\}$.

\end{itemize}

\paragraph{Trajectory-based classifier by Transformer.}
A Transformer encoder is employed for direct pathwise classification.
Each one-dimensional trajectory observation is projected into a $d_{\mathrm{model}}$-dimensional embedding, with trainable positional embeddings added to encode temporal order.
The encoded features are averaged over time and passed through layer normalization and a linear classification classifier.

The same validation split, Optuna-based tuning procedure, early-stopping rule, and random-seed setting are used as for the feedforward neural network classifier. 
The hyperparameter search space is defined as follows:

\begin{itemize}

    \item learning rate $\in \{10^{-4},\, 3\times10^{-4},\, 10^{-3},\, 3\times10^{-3}\}$,

    \item weight decay $\in\{0,\, 10^{-5},\, 10^{-4},\, 10^{-3}\}$,

    \item embedding dimension $d_{\mathrm{model}}\in\{32,\,64,\,128\}$,

    \item number of attention heads $\in\{2,\,4,\,8\}$,

    \item number of encoder layers $\in\{1,\,2,\,3\}$,

    \item feed-forward dimension $\in\{64,\,128,\,256\}$,

    \item dropout rate $\in\{0,\,0.1,\,0.2\}$,

    \item batch size $\in\{64,\,128,\,256\}$.

\end{itemize}

\paragraph{Evaluation.}
To account for the randomness in data generation and optimization process, for each pair $(N, d)$, we repeat the experiment 50 times with different random seeds, resulting in a series of different estimators and test set $(\widehat{g}_{N,d}^{(j)}, \mathcal{D}_{\text{test}}^{(j)})_{j=1}^{50}$. Let $E_{N,d}^{(j)}$ denote the generalization error computed in the $j$-th run with training size $N$ and dimension $d$, defined as 
$$E_{N,d}^{(j)} := \frac{1}{3000}\sum_{n'=1}^{1000} \sum_{k=1}^{3}\mathbbm{1}_{\big\{\widehat{g}_{N,d}^{(j)}(\tilde X^{[k], n'}_{t_0:t_M} )\neq k \big\}}-\calR(g^*),$$
%$$E_{N,d}^{(j)} := \mathcal{R}(\widehat{g}_{N,d}^{(j)}; \mathcal{D}_{\mathrm{test}}^{(j)}),$$
%where the error is evaluated on the test dataset $\mathcal{D}_{\text{test}}^{(j)}$.  
For each $(N, d)$, we compute:
\begin{equation}
\label{eq:minmax-interval}
\bar{E}_{N,d} := \frac{1}{50} \sum_{j=1}^{50} E_{N,d}^{(j)}, \quad
E_{N,d}^{\text{lower}} := \bar{E}_{N,d} - t_{0.975}^{(49)} \cdot \tau_{N,d}, \quad
E_{N,d}^{\text{upper}} := \bar{E}_{N,d} + t_{0.975}^{(49)} \cdot \tau_{N,d}.
\end{equation}
where $\tau_{N,d} = \frac{1}{\sqrt{50}} \big( \frac{1}{49} \sum_{j=1}^{50} \big(E_{N,d}^{(j)} - \bar{E}_{N,d} \big)^2 \big)^{1/2}$ and $t_{0.975}^{(49)}$ is the 97.5 percentile of the Student’s $t$ distribution with 49 degrees of freedom. 

We use log$_2$--log$_2$ plots to display the average empirical error $\bar{E}_{N,d}$ and the corresponding interval $[E_{N,d}^{\text{lower}}, E_{N,d}^{\text{upper}}]$ as functions of the training sample size $N$.

\paragraph{Compute resources.}
The experiments were run on cloud GPU servers, primarily using NVIDIA RTX-series GPUs. Across all reported experiments, the total runtime was approximately 30 GPU-hours. Individual runtimes varied across simulation settings and GPU availability.

\subsection{Example in Section \ref{sec:cos-example}.}\label{app:Experimental-3.2}
%For the comparison with B-spline estimator, we adopt the same construction method of B-spline estimator as in \citet{Denis2024}. The NN-based and 

%The B-spline–based classifiers 

%follow the same experimental setup as in Section~\ref{sec:double-layer-example} and \citet[Section~3 and 5.2]{Denis2024}, respectively.

\paragraph{Data Generation.}

In this example, the numbers of trajectories in each class are allowed to differ, following the experimental setup of \citet{Denis2024}. Specifically, we consider $N \in \{100,1000\}$ and draw $(N_1, N_2, N_3)$ from a multinomial distribution with $N$ trials and equal class probabilities $(\tfrac{1}{3}, \tfrac{1}{3}, \tfrac{1}{3})$, and repeat each experiment 100 times with different random seeds. 

Once the class sizes $N_k$ are determined, for each class $k = 1,2,3$, we generate $N_k$ independent trajectories for training
\[
\mathcal{D}^{[k]}_{N_k}
:=
\bigl\{
\bar{X}^{[k],(n)}_{t_0:t_M}
\bigr\}_{n=1}^{N_k},
\]
using the same simulation procedure as described in Section~\ref{app:Experimental-3.1}. The test set is generated similarly, with $N'=1000$ trajectories. The class prior probabilities $\mathfrak{p}_k$ are estimated empirically from the training data using
\[
\widehat{\mathfrak{p}}_k = \tfrac{N_k}{N}. 
\]

%The training set is then defined by 

%In this example, the number of paths in each class is different following the description  in \citet{Denis2024}. In particuler, we test $N \in \{100,1000\}$ and fix $(N_1, N_2, N_3)$ with the Multinomial distribution with $N$ trials and repective probabilities $(\tfrac{1}{3},\tfrac{1}{3},\tfrac{1}{3})$. 
%Once $N_k$ is fixed, for each class $k=1,2,3$, we generate $N_k$ independent trajectories $\mathcal{D}^{[k]}_{N_k} := \{ \bar{X}^{[k],(n)}_{t_0,t_M} \}_{n=1}^{N_k}$ with the same method as in Section \ref{app:Experimental-3.1}

%where  and two test sets $\mathcal{D}_{\mathrm{test}}^{[k]} := \{\tilde{X}^{[k],(n)}\}_{n=1}^{N'_k}$ with
%\[
%(N'_1, N'_2, N'_3) \sim \mathrm{Multinomial}\bigl(N',(1/3,1/3,1/3)\bigr),
%\]
%where $N' = 1000$. The classifier is then evaluated on
%\[
%\mathcal{D}_{\mathrm{test}} := \bigcup_{k\in\mathcal{Y}} \mathcal{D}_{\mathrm{test}}^{[k]}.
%\]

%We simulate sample paths of \eqref{eq:sde} under the same configuration of \ref{app:Experimental-3.1}

%However, in this example, 

\paragraph{NN-based and B-spline–based plug-in classifier.}

The configurations of the NN-based and B-spline–based plug-in classifiers are the same as those described in Section~\ref{app:Experimental-3.1}.  %We  as described in \citet{Denis2024} and compute the  

%For the B-spline estimator for $b_k$, we adopt the same procedure (including tuning and implementation details) as in \citet{Denis2024}.

%\subsubsection{Evaluation}
%We evaluate the empirical classification error of the Bayes classifier defined in \eqref{eq:def-misclassification-risk}, as well as the average empirical classification error over 100 independent runs of the NN-based and B-spline-based plug-in classifiers, for
%\[
%\theta \in \{0.5,\, 1,\, 1.5,\, 2.5,\, 4\},
%\]
%with $b_k$ specified in \eqref{eq:numerical-def-b_k-cos}.

% The empirical classification error defined in \eqref{eq:def-misclassification-risk} of Bayes classifier and the averaged empirical classification error over 100 runs of NN-based plug-in classifer and B-spline-based plug-in classifier are calculated for $\theta \in {0.5,1,5,2,5,4}$  in \eqref{eq:numerical-def-b_k-cos}.

\section{Extension to Unknown Diffusion Coefficient $\sigma$}\label{app:sigma-unknown}
%\paragraph{}

In this section, we extend the analysis in Theorem \ref{thm:main-decomposition} to the case where $\sigma$ is unknown and replaced by an estimator $\widehat{\sigma}$. We refer to e.g. \citet{Comte2007, Comte2010, Florens-Zmirou1993} and \citet{MR3052408, MR1693608} for the construction of an estimator of the diffusion coefficient $\sigma$. Recall that Assumption~\ref{assump:main-theorem} states that there exists a constant $\Lambda > 0$ such that the coefficient function $\sigma$ and 
$
a=\sigma\sigma^{\top}
$ satisfy $\max(|\sigma(x)|,|a(x)^{-1}|_{\mathrm{op}}) \leq \Lambda, \,x\in\RD$. 
Define
$
\widehat{a}\coloneqq \widehat{\sigma}\widehat{\sigma}^{\top}.
$
\begin{assumption}\label{assump:sigma-hat}
The estimator $\widehat\sigma(x)$ is invertible for every $x\in\RD$ and satisfies
\[
\forall x\in\RD,\qquad \max(|\widehat\sigma(x)|, |\widehat{a}(x)^{-1}|_{\mathrm{op}}) \le \Lambda.
\]
\end{assumption}
This assumption is natural in view of Assumption~\ref{assump:main-theorem} when \(\widehat\sigma\) is intended to provide a sufficiently accurate approximation of \(\sigma\). Now define the estimation error for $\sigma$ by
\begin{equation}\label{eq:def-estimation-error-sigma}
\mathcal{E}(\widehat \sigma, \sigma)
\coloneqq
\EE \!\left[
\frac{1}{M}\sum_{m=0}^{M-1}
\Big|\widehat \sigma(X_{t_m})-\sigma(X_{t_m})\Big|_{\mathrm{op}}^2
\right]
\end{equation}
and 
\begin{align}\label{eq:F-hat-new}
\widehat{F}^{\,\widehat{\sigma}}_k(X)
\coloneqq
&\sum_{m=0}^{M-1} \widehat{b}_k(X_{t_m})^{\top} \big(\widehat{\sigma}\widehat{\sigma}^{\top}\big)^{-1}(X_{t_m})\,(X_{t_{m+1}}-X_{t_m})
-\frac{\Delta}{2}\sum_{m=0}^{M-1} \big\lvert \widehat{\sigma}^{-1}(X_{t_m})\,\widehat{b}_k(X_{t_m}) \big\rvert^2.
\end{align}
Then, the plug-in classifier defined in \eqref{eq:def-pi-hat-g-hat} becomes
\begin{equation}\label{eq:def-pi-hat-g-hat-with-F-new}
\widehat\pi^{\,\widehat{\sigma}}_k(X) = \phi_k\big( \widehat{F}^{\,\widehat{\sigma}}(X) \big),
\qquad
\widehat g^{\,\widehat{\sigma}}(X)\in\argmax_{k\in\mathcal{Y}} \widehat \pi^{\,\widehat{\sigma}}_k(X),
\end{equation}
where
\[
\widehat{F}^{\,\widehat{\sigma}}(X)=\big(\widehat{F}^{\,\widehat{\sigma}}_1(X),\dots,\widehat{F}^{\,\widehat{\sigma}}_K(X)\big),
\]
and  the functions $\phi_k$, $k\in\mathcal{Y}$, are the softmax functions defined in \eqref{eq:def-phi}.

\begin{theorem}\label{thm:main-decomposition-sigma-unknown}
For each \( k \in \mathcal Y \), let \( \widehat b_k \) be an estimator of the drift function \( b_k \) such that 
$\sup_{x\in\RD} |\widehat b_k(x)|\leq \widehat b_{\max}$ for some constant $\widehat b_{\max}>0$.  Let \(\widehat\sigma\) be an estimator of \(\sigma\), and let \(g^{\,\widehat\sigma}\) denote the plug-in classifier defined in \eqref{eq:def-pi-hat-g-hat-with-F-new}, associated with the estimators \(\widehat b_k, {k\in\calY}\) and \(\widehat\sigma\). Assume that Assumptions~\ref{assump:Lipschitz}, \ref{assump:Novikov}, \ref{assump:main-theorem} and \ref{assump:sigma-hat} hold.
Then there exists a constant \( C_{\Lambda, \mathfrak{C}, b_{\max}} > 0 \), such that
\begin{equation}\label{eq:inequality-in-main-thm-with-sigma-hat}
    \calR(g^{\,\widehat{\sigma}})-\mathcal{R}(g^*)\leq K^2 C_{\Lambda, \mathfrak{C}, \widehat b_{\max}}\big(\sqrt{\Delta} + \max_{k\in\calY}\calE(\widehat b_k, b_k)^{\frac{1}{2}}+\mathcal{E}(\widehat\sigma,\sigma)^{\frac{1}{2}}\big).
\end{equation}
\end{theorem}

The proof of Theorem \ref{thm:main-decomposition-sigma-unknown} is based on the following proposition and is postponed to the end of this section.

\begin{proposition}\label{prop:error-F-hat-new}
Under the same conditions as in Theorem \ref{thm:main-decomposition-sigma-unknown},  there exists a  constant \(C=C_{\Lambda,\mathfrak C,\widehat b_{\max}}\) such that
\[ \EE \left[ \left|\bar F_k(X)-\widehat F^{\,\widehat{\sigma}}_k(X)\right|\right]\leq C_{\Lambda, \mathfrak{C}, \widehat b_{\max}}\left(\calE(\widehat b_k, b_k)^{\frac{1}{2}}+\mathcal{E}(\widehat\sigma,\sigma)^{\frac{1}{2}}\right).\]
\end{proposition}

The proof of Proposition \ref{prop:error-F-hat-new} relies on the following two lemmas.

\begin{lemma}\label{lem:sigma-hat-bound}
    Assume that Assumptions \ref{assump:Lipschitz}, \ref{assump:Novikov}, \ref{assump:main-theorem} and \ref{assump:sigma-hat} hold. Then, for every \(x\in\RD\),
\begin{align}
&\left| b_k(x)^{\top}a^{-1}(x)-\widehat b_k(x)^{\top}\,\widehat{a}^{-1}(x)\right|\leq \Lambda |b_k(x)-\widehat b_k(x)|+\Lambda^2|\widehat{b}_k(x)|\left|\widehat{a}(x)-a(x)\right|_{\mathrm{op}},\nonumber\\
&\Big|\big\lvert \sigma^{-1} (x)\,b_k(x) \big\rvert^2-\big\lvert \widehat{\sigma}^{-1} (x)\,\widehat{b}_k(x) \big\rvert^2\Big|\nonumber\\
&\qquad \leq \Lambda |b_k(x)-\widehat b_k(x)|\bigl(|b_k(x)|+|\widehat b_k(x)|\bigr)+\Lambda^2 |\widehat{b}_k(x)|^2\,|\widehat a(x)-a(x)|_{\mathrm{op}}.\nonumber
\end{align}
\end{lemma}

\begin{proof}[Proof of Lemma \ref{lem:sigma-hat-bound}]
First, note that for any invertible matrices \(A\) and \(B\), we have
\[
A^{-1}-B^{-1}=A^{-1}(B-A)B^{-1}.
\]
Therefore,
\begin{equation}%\label{eq:operator-norm-submultiplicative}
   |A^{-1}-B^{-1}|_{\mathrm{op}}\leq|A^{-1}|_{\mathrm{op}}|B-A|_{\mathrm{op}}|B^{-1}|_{\mathrm{op}}\nonumber
\end{equation}
by the submultiplicativity of the operator norm.
Hence, for every \(x\in\RD\),
%\begin{align}
%&\left| b_k(x)^{\top}a^{-1}(x)-\widehat b_k(x)^{\top}\,\widehat{a}^{-1}(x)\right|\nonumber\\
%&\leq \left| b_k(x)^{\top}a^{-1}(x)- b_k(x)^{\top}\,\widehat{a}^{-1}(x)\right|+\left| b_k(x)^{\top}\widehat{a}^{-1}(x)-\widehat b_k(x)^{\top}\,\widehat{a}^{-1}(x)\right|\nonumber\\
%&\leq |b_k(x)||a^{-1}(x)-\widehat{a}^{-1}(x)|_{\mathrm{op}}+|b_k(x)-\widehat b_k(x)||\widehat{a}^{-1}(x)|_{\mathrm{op}}\nonumber\\
%&\leq  |b_k(x)||a^{-1}(x)|_{\mathrm{op}}|\widehat{a}(x)-a(x)|_{\mathrm{op}}|\widehat{a}^{-1}(x)|_{\mathrm{op}}+|b_k(x)-\widehat b_k(x)||\widehat{a}^{-1}(x)|_{\mathrm{op}}\nonumber\\
%&\leq \Lambda^2|b_k(x)||\widehat{a}(x)-a(x)|_{\mathrm{op}}+\Lambda|b_k(x)-\widehat b_k(x)|.\nonumber
%\end{align}
\begin{align}
&\left| b_k(x)^{\top}a^{-1}(x)-\widehat b_k(x)^{\top}\,\widehat{a}^{-1}(x)\right|\nonumber\\
&\leq \left| b_k(x)^{\top}a^{-1}(x)- \widehat{b}_k(x)^{\top}\,a^{-1}(x)\right|+\left| \widehat{b}_k(x)^{\top}\,a^{-1}(x)-\widehat b_k(x)^{\top}\,\widehat{a}^{-1}(x)\right|\nonumber\\
&\leq |b_k(x)-\widehat b_k(x)||a^{-1}(x)|_{\mathrm{op}}+|\widehat{b}_k(x)||a^{-1}(x)-\widehat{a}^{-1}(x)|_{\mathrm{op}}\nonumber\\
&\leq |b_k(x)-\widehat b_k(x)||{a}^{-1}(x)|_{\mathrm{op}}+ |\widehat{b}_k(x)||a^{-1}(x)|_{\mathrm{op}}|\widehat{a}(x)-a(x)|_{\mathrm{op}}|\widehat{a}^{-1}(x)|_{\mathrm{op}}\nonumber\\
&\leq \Lambda|b_k(x)-\widehat b_k(x)|+\Lambda^2|\widehat{b}_k(x)||\widehat{a}(x)-a(x)|_{\mathrm{op}}.\nonumber
\end{align}

We now prove the second inequality. For every \(x\in\RD\),
%\begin{align}
%&\Big|\big\lvert \sigma^{-1} (x)\,b_k(x) \big\rvert^2-\big\lvert \widehat{\sigma}^{-1} (x)\,\widehat{b}_k(x) \big\rvert^2\Big|= \left| b_k(x)^{\top}a^{-1}(x)b_k(x)-\widehat b_k(x)^{\top}\widehat a^{-1}(x)\widehat b_k(x)\right|\nonumber\\
%&\leq \left| b_k(x)^{\top}a^{-1}(x)b_k(x)-b_k(x)^{\top}\widehat a^{-1}(x)b_k(x)\right| +\left| b_k(x)^{\top}\widehat a^{-1}(x)b_k(x)-b_k(x)^{\top}\widehat a^{-1}(x)\widehat b_k(x)\right|\nonumber\\
%&\quad +\left| b_k(x)^{\top}\widehat a^{-1}(x)\widehat b_k(x)-\widehat b_k(x)^{\top}\widehat a^{-1}(x)\widehat b_k(x)\right|\nonumber\\
%&\leq |b_k(x)|^2\,|a^{-1}(x)-\widehat a^{-1}(x)|_{\mathrm{op}} +|b_k(x)|\,|\widehat a^{-1}(x)|_{\mathrm{op}}\,|b_k(x)-\widehat b_k(x)| +|\widehat b_k(x)|\,|\widehat a^{-1}(x)|_{\mathrm{op}}\,|b_k(x)-\widehat b_k(x)|\nonumber\\
%&\leq |b_k(x)|^2\,|a^{-1}(x)|_{\mathrm{op}}\,|\widehat a(x)-a(x)|_{\mathrm{op}}\,|\widehat a^{-1}(x)|_{\mathrm{op}}+\Lambda |b_k(x)-\widehat b_k(x)|\bigl(|b_k(x)|+|\widehat b_k(x)|\bigr)\nonumber\\
%&\leq \Lambda^2 |b_k(x)|^2\,|\widehat a(x)-a(x)|_{\mathrm{op}}
%+\Lambda |b_k(x)-\widehat b_k(x)|\bigl(|b_k(x)|+|\widehat b_k(x)|\bigr).\nonumber\hfill \qedhere
%\end{align}
\begin{align}
&\Big|\big\lvert \sigma^{-1} (x)\,b_k(x) \big\rvert^2-\big\lvert \widehat{\sigma}^{-1} (x)\,\widehat{b}_k(x) \big\rvert^2\Big|= \left| b_k(x)^{\top}a^{-1}(x)b_k(x)-\widehat b_k(x)^{\top}\widehat a^{-1}(x)\widehat b_k(x)\right|\nonumber\\
&\leq \left| b_k(x)^{\top}a^{-1}(x)b_k(x)-\widehat b_k(x)^{\top} a^{-1}(x)b_k(x)\right| +\left| \widehat b_k(x)^{\top} a^{-1}(x)b_k(x)-\widehat b_k(x)^{\top} a^{-1}(x)\widehat b_k(x)\right|\nonumber\\
&\quad +\left| \widehat b_k(x)^{\top} a^{-1}(x)\widehat b_k(x)-\widehat b_k(x)^{\top}\widehat a^{-1}(x)\widehat b_k(x)\right|\nonumber\\
&\leq |b_k(x)|\,| a^{-1}(x)|_{\mathrm{op}}\,|b_k(x)-\widehat b_k(x)| +|\widehat b_k(x)|\,| a^{-1}(x)|_{\mathrm{op}}\,|b_k(x)-\widehat b_k(x)|\nonumber\\
&\qquad +|\widehat{b}_k(x)|^2\,|a^{-1}(x)-\widehat a^{-1}(x)|_{\mathrm{op}} \nonumber\\
&\leq \Lambda |b_k(x)-\widehat b_k(x)|\bigl(|b_k(x)|+|\widehat b_k(x)|\bigr)+|\widehat{b}_k(x)|^2\,|a^{-1}(x)|_{\mathrm{op}}\,|\widehat a(x)-a(x)|_{\mathrm{op}}\,|\widehat a^{-1}(x)|_{\mathrm{op}}\nonumber\\
&\leq \Lambda |b_k(x)-\widehat b_k(x)|\bigl(|b_k(x)|+|\widehat b_k(x)|\bigr)+\Lambda^2 |\widehat{b}_k(x)|^2\,|\widehat a(x)-a(x)|_{\mathrm{op}}.\nonumber\hfill \qedhere
\end{align}
\end{proof}

\begin{lemma}\label{lem:a-hat-sigma-hat-bound}
Assume that Assumptions~\ref{assump:main-theorem} and \ref{assump:sigma-hat} hold. We have
\[
\EE\!\left[
\frac{1}{M}\sum_{m=0}^{M-1}
\left|\widehat a(X_{t_m})-a(X_{t_m})\right|_{\mathrm{op}}^2
\right]
\le 4\Lambda^2\,\mathcal E(\widehat \sigma,\sigma).
\]
\end{lemma}

\begin{proof}[Proof of Lemma \ref{lem:a-hat-sigma-hat-bound}]
For every \(x\in\RD\), we write
\begin{align*}
\widehat a(x)-a(x)
&=
\widehat \sigma(x)\widehat \sigma(x)^{\top}
-\sigma(x)\sigma(x)^{\top}=
\bigl(\widehat \sigma(x)-\sigma(x)\bigr)\widehat \sigma(x)^{\top}
+\sigma(x)\bigl(\widehat \sigma(x)-\sigma(x)\bigr)^{\top}.
\end{align*}
Taking the operator norm and using the triangle inequality, together with the submultiplicativity of the operator norm, yields
\begin{align*}
\left|\widehat a(x)-a(x)\right|_{\mathrm{op}}
&\le
\left|\bigl(\widehat \sigma(x)-\sigma(x)\bigr)\widehat \sigma(x)^{\top}\right|_{\mathrm{op}}
+
\left|\sigma(x)\bigl(\widehat \sigma(x)-\sigma(x)\bigr)^{\top}\right|_{\mathrm{op}}\\
&\le
\left|\widehat \sigma(x)-\sigma(x)\right|_{\mathrm{op}}
\left|\widehat \sigma(x)^{\top}\right|_{\mathrm{op}}
+
\left|\sigma(x)\right|_{\mathrm{op}}
\left|\bigl(\widehat \sigma(x)-\sigma(x)\bigr)^{\top}\right|_{\mathrm{op}}.
\end{align*}
%Since the operator norm is invariant under transposition, we have
%\[
%\left|\widehat \sigma(x)^{\top}\right|_{\mathrm{op}}
%=
%\left|\widehat \sigma(x)\right|_{\mathrm{op}},
%\qquad
%\left|\bigl(\widehat \sigma(x)-\sigma(x)\bigr)^{\top}\right|_{\mathrm{op}}
%=
%\left|\widehat \sigma(x)-\sigma(x)\right|_{\mathrm{op}}.
%\]
Therefore,
\[
\left|\widehat a(x)-a(x)\right|_{\mathrm{op}}
\le
\bigl(
|\widehat \sigma(x)|_{\mathrm{op}}+|\sigma(x)|_{\mathrm{op}}
\bigr)
\left|\widehat \sigma(x)-\sigma(x)\right|_{\mathrm{op}}\le 2\Lambda \left|\widehat \sigma(x)-\sigma(x)\right|_{\mathrm{op}}
\]
and
\[
\EE\!\left[
\frac{1}{M}\sum_{m=0}^{M-1}
\left|\widehat a(X_{t_m})-a(X_{t_m})\right|_{\mathrm{op}}^2
\right]
\le
4\Lambda^2
\EE\!\left[
\frac{1}{M}\sum_{m=0}^{M-1}
\left|\widehat \sigma(X_{t_m})-\sigma(X_{t_m})\right|_{\mathrm{op}}^2
\right]= 4\Lambda^2\,\mathcal E(\widehat \sigma,\sigma). \hfill\qedhere
\]
%The last claim follows immediately from the definition of \(\mathcal E(\widehat \sigma,\sigma)\).
\end{proof}

%\begin{lemma}\label{lem:sigma-hat-bound}
%Assume that Assumptions~\ref{assump:Lipschitz}, \ref{assump:Novikov}, \ref{assump:main-theorem}, and \ref{assump:sigma-hat} hold. Then, for every \(x\in\RD\),

%\end{lemma}

\begin{proof}[Proof of Proposition \ref{prop:error-F-hat-new}]
Recall from \eqref{eq:def-eta(s)} that, for every $s\in[0,T]$,
\[
\eta(s)\coloneqq t_m \quad \text{if } t_m \le s < t_{m+1}, \qquad \text{with } \eta(T)=T.
\]
Hence, we can rewrite
\begin{align}\label{eq:F-bar-rewrite}
&\bar{F}_k(X)
= \int_{0}^{T}b_k(X_{\eta(s)})^{\top} \big(\sigma\sigma^{\top}\big)^{-1}(X_{\eta(s)})\d X_s - \frac{1}{2}\int_{0}^{T} \big\lvert \sigma^{-1} (X_{\eta(s)})\,b_k(X_{\eta(s)}) \big\rvert^2 \d s,\nonumber\\
&\widehat{F}^{\,\widehat{\sigma}}_k(X)=\int_{0}^{T} \widehat{b}_k(X_{\eta(s)})^{\top} \big(\widehat{\sigma}\widehat{\sigma}^{\top}\big)^{-1}(X_{\eta(s)})\d X_s - \frac{1}{2}\int_{0}^{T} \big\lvert \widehat{\sigma}^{-1} (X_{\eta(s)})\,\widehat{b}_k(X_{\eta(s)}) \big\rvert^2 \d s.
\end{align}
Therefore, 
\begin{align}\label{eq:part-I-and-II-add-sigma-hat}
    &\EE \left[ |\bar F_k(X)-\widehat F^{\,\widehat{\sigma}}_k(X)|\right]\leq  \EE \left[\left|\int_{0}^T \Big( b_k(X_{\eta(s)})^{\top}a^{-1}(X_{\eta(s)})-\widehat b_k(X_{\eta(s)})^{\top}\,\widehat{a}^{-1}(X_{\eta(s)})\Big)\d X_s\right|\right]\nonumber\\
     &\quad +\frac{1}{2}\EE \left[\left|\int_0^T \left[\big\lvert \sigma^{-1} (X_{\eta(s)})\,b_k(X_{\eta(s)}) \big\rvert^2-\big\lvert \widehat{\sigma}^{-1} (X_{\eta(s)})\,\widehat{b}_k(X_{\eta(s)}) \big\rvert^2\right]\d s\right|\right]\eqqcolon \text{(I) + $\frac{1}{2}$(II)}.
\end{align}
For part (I) of \eqref{eq:part-I-and-II-add-sigma-hat}, we have
\begin{align}\label{eq:I}
 \text{(I)} & =   \EE \left[\left|\int_{0}^T \Big( b_k(X_{\eta(s)})^{\top}a^{-1}(X_{\eta(s)})-\widehat b_k(X_{\eta(s)})^{\top}\,\widehat{a}^{-1}(X_{\eta(s)})\Big)\Big(b_Y(X_s) \d s + \sigma(X_s) \d B_s \Big)\right|\right]\nonumber\\
 & \leq \EE \left[\left|\int_{0}^T \Big( b_k(X_{\eta(s)})^{\top}a^{-1}(X_{\eta(s)})-\widehat b_k(X_{\eta(s)})^{\top}\,\widehat{a}^{-1}(X_{\eta(s)})\Big)b_Y(X_s) \d s \right|\right]\nonumber\\
 &\quad +\EE \left[\left|\int_{0}^T \Big( b_k(X_{\eta(s)})^{\top}a^{-1}(X_{\eta(s)})-\widehat b_k(X_{\eta(s)})^{\top}\,\widehat{a}^{-1}(X_{\eta(s)})\Big)\sigma(X_s) \d B_s \right|\right]\eqqcolon \text{(I.a) + (I.b)}.
\end{align}
Next, for (I.a), using  Lemma \ref{lem:properties-X} and Lemma \ref{lem:sigma-hat-bound},  we obtain 
\begin{align}\label{eq:I-a}
    \text{(I.a)}&\leq \EE \left[\int_{0}^T \left|\Big( b_k(X_{\eta(s)})^{\top}a^{-1}(X_{\eta(s)})-\widehat b_k(X_{\eta(s)})^{\top}\,\widehat{a}^{-1}(X_{\eta(s)})\Big)b_Y(X_s) \right|\d s \right]\nonumber\\
    &\leq \EE \left[\int_{0}^T \left| b_k(X_{\eta(s)})^{\top}a^{-1}(X_{\eta(s)})-\widehat b_k(X_{\eta(s)})^{\top}\,\widehat{a}^{-1}(X_{\eta(s)})\right| \left|b_Y(X_s) \right|\d s \right]\nonumber\\
 %   &\leq \EE \left[\int_{0}^T \left[ \Lambda \big|b_k(X_{\eta(s)})-\widehat b_k(X_{\eta(s)})\big|\left|b_Y(X_s) \right|+\Lambda^2|\widehat{b}_k(X_{\eta(s)})|\left|\widehat{a}(X_{\eta(s)})-a(X_{\eta(s)})\right|_{\mathrm{op}}\left|b_Y(X_s) \right|\right] \d s \right]\nonumber\\
    &\leq \Lambda\EE \left[\int_{0}^T  \big|{b}_k(X_{\eta(s)})-\widehat b_k(X_{\eta(s)})\big|\left|b_Y(X_s) \right| \d s \right] \nonumber\\
    &\qquad + \Lambda^2\EE \left[\int_{0}^T |\widehat{b}_k(X_{\eta(s)})|\left|\widehat{a}(X_{\eta(s)})-a(X_{\eta(s)})\right|_{\mathrm{op}}\left|b_Y(X_s) \right| \d s \right] \nonumber\\
    &\eqqcolon \text{(I.a.1)}+\text{(I.a.2)}. 
\end{align}
For \text{(I.a.1)}, using the Cauchy–Schwarz inequality,  we obtain 
\begin{align}\label{eq:I-a}
    \text{(I.a.1)}%&\leq \EE \left[\int_{0}^T \left|\big( b_k(X_{\eta(s)})-\widehat b_k(X_{\eta(s)})\big)^{\top}a^{-1}(X_{\eta(s)})b_Y(X_s) \right|\d s \right]\nonumber\\
    %&\leq \Lambda\EE \left[\int_{0}^T \left|\big( b_k(X_{\eta(s)})-\widehat b_k(X_{\eta(s)})\big)\right| \left|b_Y(X_s) \right|\d s \right]\nonumber\\
    &\leq \Lambda\EE \left[\left(\int_{0}^T \left|\big( b_k(X_{\eta(s)})-\widehat b_k(X_{\eta(s)})\big)\right|^2\d s\right)^{\frac{1}{2}}\left(\int_{0}^T\left|b_Y(X_s) \right|^2\d s\right)^{\frac{1}{2}} \right]\nonumber\\
    &\leq \Lambda \left(\EE \left[\int_{0}^T \left|\big( b_k(X_{\eta(s)})-\widehat b_k(X_{\eta(s)})\big)\right|^2\d s\right]\right)^{\frac{1}{2}}\left(\EE \left[\int_{0}^T\left|b_Y(X_s) \right|^2\d s \right]\right)^{\frac{1}{2}}\nonumber\\
    &\leq \Lambda \mathfrak{C}\left(\EE \left[\Delta\sum_{m=0}^{M-1}\left|\big( b_k(X_{t_m})-\widehat b_k(X_{t_m})\big)\right|^2\right]\right)^{\frac{1}{2}}=\Lambda \mathfrak{C}\,  \calE(\widehat b_k, b_k)^{\frac{1}{2}}.
\end{align}
Similarly, for \text{(I.a.2)}, still by the Cauchy–Schwarz inequality,  we obtain 
\begin{align}
    \text{(I.a.2)}&\leq \Lambda^2\EE \left[\left(\int_{0}^T \left|\widehat{a}(X_{\eta(s)})-a(X_{\eta(s)})\right|_{\mathrm{op}}^2\d s\right)^{\frac{1}{2}}\left(\int_{0}^T|\widehat b_k(X_{\eta(s)})|^2\left|b_Y(X_s) \right|^2\d s\right)^{\frac{1}{2}} \right]\nonumber\\
    &\leq \Lambda^2\left(\EE \left[\int_{0}^T \left|\widehat{a}(X_{\eta(s)})-a(X_{\eta(s)})\right|_{\mathrm{op}}^2\d s\right]\right)^{\frac{1}{2}}\left(\EE \left[\int_{0}^T|\widehat{b}_k(X_{\eta(s)})|^2\left|b_Y(X_s) \right|^2\d s \right]\right)^{\frac{1}{2}}\nonumber\\
    &\leq \Lambda^2 \widehat{b}_{\max}\sqrt{T}\left(\EE \left[\Delta\sum_{m=0}^{M-1}\left|\big( \widehat a(X_{t_m})-a(X_{t_m})\big)\right|_{\mathrm{op}}^2\right]\right)^{\frac{1}{2}} \left(\EE \left[\int_{0}^T\left|b_Y(X_s) \right|^2\d s \right]\right)^{\frac{1}{2}}\nonumber\\
    &\leq \Lambda^2\mathfrak{C}\left(\EE \left[\Delta\sum_{m=0}^{M-1}\left|\big( \widehat a(X_{t_m})-a(X_{t_m})\big)\right|_{\mathrm{op}}^2\right]\right)^{\frac{1}{2}}\leq \Lambda^{3}\widehat{b}_{\max}\mathfrak{C}\,\mathcal{E}(\widehat\sigma,\sigma)^{\frac{1}{2}}. \nonumber
    %&\leq \Lambda^2\mathfrak{C}\left(\EE \left[\Delta\sum_{m=0}^{M-1}\left|\big( b_k(X_{t_m})-\widehat b_k(X_{t_m})\big)\right|^2\right]\right)^{\frac{1}{2}}=\Lambda \mathfrak{C}\,  \calE(\widehat b_k, b_k)^{\frac{1}{2}}.\nonumber\\
    %&\leq \Lambda^2\EE \left[\int_{0}^T |b_k(X_{\eta(s)})|\left|\widehat{a}(X_{\eta(s)})-a(X_{\eta(s)})\right|_{\mathrm{op}}\left|b_Y(X_s) \right| \d s \right]
\end{align}
For (I.b), by applying Lemma \ref{lem:sigma-hat-bound} and using Jensen's inequality and It\^o isometry, we have
\begin{align}
  & \text{(I.b)}= \EE \left[\left|\int_{0}^T \Big( b_k(X_{\eta(s)})^{\top}a^{-1}(X_{\eta(s)})-\widehat b_k(X_{\eta(s)})^{\top}\,\widehat{a}^{-1}(X_{\eta(s)})\Big)\sigma(X_s) \d B_s \right|\right]\nonumber\\
   &\leq \EE \left[\left|\int_{0}^T \Big( \Lambda |b_k(X_{\eta(s)})-\widehat b_k(X_{\eta(s)})|+\Lambda^2|\widehat b_k(X_{\eta(s)})|\left|\widehat{a}(X_{\eta(s)})-a(X_{\eta(s)})\right|_{\mathrm{op}}\Big)\sigma(X_s) \d B_s \right|\right]\nonumber\\
%   &\leq \left\{\EE \left[\left|\int_{0}^T \Big( \Lambda |b_k(X_{\eta(s)})-\widehat b_k(X_{\eta(s)})|+\Lambda^2|\widehat b_k(X_{\eta(s)})|\left|\widehat{a}(X_{\eta(s)})-a(X_{\eta(s)})\right|_{\mathrm{op}}\Big)\sigma(X_s) \d B_s \right|^2\right]\right\}^{\frac{1}{2}}\nonumber\\
   &\leq \left\{\!\EE \left[\int_{0}^T \!\!\left|\Big( \Lambda |b_k(X_{\eta(s)})-\widehat b_k(X_{\eta(s)})|\!+\!\Lambda^2|\widehat b_k(X_{\eta(s)})|\left|\widehat{a}(X_{\eta(s)})\!-\!a(X_{\eta(s)})\right|_{\mathrm{op}}\!\Big)\sigma(X_s)\right|^2 \!\d s \right]\!\right\}^{\frac{1}{2}}\nonumber\\
   &\leq\sqrt{2}\left\{\EE \left[\int_{0}^T \left|\Lambda |b_k(X_{\eta(s)})-\widehat b_k(X_{\eta(s)})|\sigma(X_s)\right|^2 \d s \right]\right\}^{\frac{1}{2}}\nonumber\\
   &\qquad +\sqrt{2}\left\{\EE \left[\int_{0}^T \left|\Lambda^2|\widehat b_k(X_{\eta(s)})|\left|\widehat{a}(X_{\eta(s)})-a(X_{\eta(s)})\right|_{\mathrm{op}}\sigma(X_s)\right|^2 \d s \right]\right\}^{\frac{1}{2}}\nonumber\\
   &\eqqcolon \text{(I.b.1)}+\text{(I.b.2)}.\nonumber
\end{align}
Next, for (I.b.1),  we obtain
\begin{align}
    \text{(I.b.1)}\leq \sqrt{2}\Lambda^2 \left\{    \int_{0}^T \EE \left[\left| b_k(X_{\eta(s)})-\widehat b_k(X_{\eta(s)})\right|^2\right]\d s \right\}^{\frac{1}{2}} = \sqrt{2}\Lambda^2  \sqrt{T} \calE(\widehat b_k, b_k)^{\frac{1}{2}}.\nonumber
\end{align}
Similarly, for (I.b.2), we have
\begin{align}
    \text{(I.b.2)}
    &=\sqrt{2}\left\{ \EE \left[\int_{0}^T \Lambda^4|\widehat b_k(X_{\eta(s)})|^2\left|\widehat{a}(X_{\eta(s)})-a(X_{\eta(s)})\right|_{\mathrm{op}}^2|\sigma(X_s)|^2 \d s \right]\right\}^{\frac{1}{2}}\nonumber\\
    &\leq \sqrt{2}\Lambda^3 \widehat b_{\max} \mathfrak{C}\left(\EE \left[\Delta\sum_{m=0}^{M-1}\left|\big( \widehat a(X_{t_m})-a(X_{t_m})\big)\right|_{\mathrm{op}}^2\right]\right)^{\frac{1}{2}}\leq \sqrt{2}\Lambda^{4} \widehat b_{\max} \mathfrak{C}\mathcal{E}(\widehat\sigma,\sigma)^{\frac{1}{2}}.
\end{align}
For part (II) of \eqref{eq:part-I-and-II-add-sigma-hat}, by applying Lemma \ref{lem:sigma-hat-bound} and the Cauchy–Schwarz inequality, we obtain 
\begin{align}
    \text{(II)} & = \EE \left[\left|\int_0^T \left[\big\lvert \sigma^{-1} (X_{\eta(s)})\,b_k(X_{\eta(s)}) \big\rvert^2-\big\lvert \widehat{\sigma}^{-1} (X_{\eta(s)})\,\widehat{b}_k(X_{\eta(s)}) \big\rvert^2\right]\d s\right|\right]\nonumber\\
    &\leq \EE \left[\int_0^T \Lambda \left|b_k(X_{\eta(s)})-\widehat b_k(X_{\eta(s)})\right|\left(\left|b_k(X_{\eta(s)})\right|+\left|\widehat b_k(X_{\eta(s)})\right|\right)\d s\right]\nonumber\\
    &\qquad +\EE \left[\int_0^T \Lambda^2 \left|\widehat{b}_k(X_{\eta(s)})\right|^2\,\left|\widehat a(X_{\eta(s)})-a(X_{\eta(s)})\right|_{\mathrm{op}}\d s\right]\nonumber\\
    &\leq \Lambda \!\left\{\!\EE \left[\int_0^T \!\!\left|b_k(X_{\eta(s)})-\widehat b_k(X_{\eta(s)})\right|^2 \!\d s\right]\!\right\}^{\frac{1}{2}}\!\!\left\{\EE\left[\int_0^T \!\!\left(\left|b_k(X_{\eta(s)})\right|+\left|\widehat b_k(X_{\eta(s)})\right|\right)^2\!\d s\right]\!\right\}^{\frac{1}{2}}\nonumber\\
    &\qquad+\Lambda^2 \widehat{b}_{\max}^2\left\{\EE \left[\int_0^T \left|\widehat a(X_{\eta(s)})-a(X_{\eta(s)})\right|_{\mathrm{op}}^2\d s\right]\right\}^{\frac{1}{2}}\nonumber\\
    &\leq \Lambda (\sqrt{2}\widehat{b}_{\max}+\mathfrak{C})\calE(\widehat b_k, b_k)^{\frac{1}{2}}+\Lambda^{3} \widehat{b}_{\max}^2\mathfrak{C}\mathcal{E}(\widehat\sigma,\sigma)^{\frac{1}{2}}.
\end{align}

Combining the above bounds yields
   \[\EE \left[ |\bar F_k(X)-\widehat F^{\,\widehat{\sigma}}_k(X)|\right]\leq C_{\Lambda, \mathfrak{C}, \widehat b_{\max}}\left(\calE(\widehat b_k, b_k)^{\frac{1}{2}}+\mathcal{E}(\widehat\sigma,\sigma)^{\frac{1}{2}}\right). \hfill\qedhere\]
%%%%%%%%
%%%%%%%%
%%%%%%%%
\end{proof}

\begin{proof}[Proof of Theorem \ref{thm:main-decomposition-sigma-unknown}]
The first part of the proof follows the same arguments as in Theorem \ref{thm:main-decomposition}, with $\widehat \pi_k(X)$ and $\widehat g$ replaced by $\widehat \pi_k^{\,\widehat{\sigma}}(X)$ and $\widehat g^{\,\widehat{\sigma}}$, respectively. Using the definition of $\widehat{F}^{\,\widehat{\sigma}}$, inequality \eqref{eq:ineq-pi-and-F} in the proof of Theorem~\ref{thm:main-decomposition} becomes
\begin{align}\label{eq:ineq-with-F-new}
\EE \big[ \big|\pi_k^*(X)-\widehat \pi_k^{\,\widehat{\sigma}}(X) \big|\big]
&\leq \EE \big[ \big|\pi_k^*(X)-\bar \pi_k(X) \big|\big]
   +\EE \big[ \big|\bar\pi_k(X)-\widehat \pi^{\,\widehat{\sigma}}_k(X) \big|\big] \nonumber\\
&\leq \EE \big[ \big|F^*(X)-\bar F(X) \big|\big]
   +\EE \big[ \big|\bar F(X)-\widehat F^{\,\widehat{\sigma}}(X) \big|\big] \nonumber\\
&\leq \sum_{k=1}^K \left(
\EE \big[ \big|F_k^*(X)-\bar F_k(X) \big|\big]
+\EE \big[ \big|\bar F_k(X)-\widehat F^{\,\widehat{\sigma}}_k(X) \big|\big]
\right).
\end{align}
The term
$
\EE \big[ \big|F_k^*(X)-\bar F_k(X) \big|\big]
$
is treated exactly as in Step~1 of the proof of Theorem~\ref{thm:main-decomposition}. Moreover, the upper bound for
$
\EE \big[ \big|\bar F_k(X)-\widehat F^{\,\widehat{\sigma}}_k(X) \big|\big]
$
follows from Proposition~\ref{prop:error-F-hat-new}, which concludes the proof.
\end{proof}

\section{Comparison with Trajectory-Based Classifiers}\label{app:trajectory-comparison}

In this section, we compare the convergence rate of our approach with the trajectory-based classifier within the framework of \citet{bos2022}. We note that \citet{bos2022} studies classification error in terms of the Kullback--Leibler risk with respect to the true conditional class probabilities, rather than the excess risk considered in Theorems~\ref{thm:main-decomposition} and~\ref{thm:main-NN}. Hence, we first need to establish a connection between these two measures of classification error before discussing the comparison.

For the reader's convenience, we first recall the main notation of \citet{bos2022} in Section~\ref{sec:bos-notation}. In Section~\ref{sec:bos-prop}, we derive in Proposition~\ref{prop:excess-risk-bos2022} a convergence rate for the excess risk based on Theorem~3.3 of \citet{bos2022}. Finally, Section~\ref{sec:bos-comparison} provides a detailed comparison between Proposition~\ref{prop:excess-risk-bos2022} and Theorem~\ref{thm:main-NN}.

%For the reader's convenience, we first recall the main notation introduced in \citet{bos2022} in Section \ref{sec:bos-notation}. Then, Proposition~\ref{prop:excess-risk-bos2022} in Section \ref{sec:bos-prop} derives a convergence rate for the excess risk based on Theorem~3.3 of \citet{bos2022}. Finally, we  provide a detailed comparison between Proposition~\ref{prop:excess-risk-bos2022} and Theorem~\ref{thm:main-NN} in Section \ref{sec:bos-comparison}.

\subsection{Notation from \citet{bos2022}.} \label{sec:bos-notation}

\paragraph{Data, true conditional class probabilities, and error metric.} The dataset $\calD_N=\{(\mathbf{X}_i, \mathbf{Y}_i):i=1, ..., N\}$ consists of $n$ i.i.d. copies of the pair $(\mathbf{X}, \mathbf{Y})$, where the input data $\mathbf{X}$ takes values in $[0,1]^\mathbf{d}$ and the corresponding label
$\mathbf{Y}=(Y_1,\ldots,Y_K)^\top$ is encoded using the $K$-dimensional standard basis vectors: namely, $\mathbf{Y}=e_k=(0, \cdots, 0,\underset{k\text{-th}}{\underbrace{1}},0,\cdots, 0)$, if the observation belongs to class $k$. The true conditional class probabilities are defined as\begin{equation}\label{eq:true-cond-bos2022}    p_k^0(\mathbf{x})\coloneqq \mathbb{P}(Y_k=1\mid \mathbf{X}=\mathbf{x}), \qquad k=1,\ldots,K,\end{equation}
and we denote by\begin{equation}\label{eq:true-cond-vector-bos2022} 
     \mathbf{p}_0(\mathbf{x})\coloneqq \big(p_1^0(\mathbf{x}),\ldots,p_K^0(\mathbf{x})\big)^\top
\end{equation}  the corresponding vector of conditional class probabilities. In \citet{bos2022}, the expected truncated Kullback--Leibler risk 
\begin{equation}\label{eq:bos-error-def-RB}
    R_B(\mathbf{p}_0, \widehat{\mathbf{p}})\coloneqq \EE_{\calD_{N}, \mathbf{X}}\left[\,\mathrm{KL}_B\big( \mathbf{p}_0(\mathbf{X}), \widehat{\mathbf{p}}(\mathbf{X})\big)\,\right]
\end{equation}
is used as a metric to measure the discrepancy between the true conditional probabilities $\mathbf{p}_0$ and their estimators $\widehat{\mathbf{p}}$, where
\begin{equation}\label{eq:def-bos-truncated-Kullback--Leibler}
\mathrm{KL}_B \big(\mathbf{p}_0(\mathbf{X}), \widehat{\mathbf{p}}(\mathbf{X})\big) \coloneqq \mathbf{p}_0(\mathbf{X})^{\top} \left( B\wedge \log \left( \frac{\mathbf{p}_0(\mathbf{X})}{\widehat{\mathbf{p}}(\mathbf{X})}\right)\right),
\end{equation}
and $B$ is the truncation parameter (see Section 2 in \citet{bos2022}). When $B=\infty$, the truncated Kullback--Leibler risk $\mathrm{KL}_B$ coincides with the standard Kullback--Leibler risk.

\paragraph{FNN classifier in \citet{bos2022} and assumptions.} Let $\mathcal{S}_K$ be the $(K-1)$-simplex in $\RR^K$, that
is, \[\mathcal{S}_K=\left\{\mathbf{v}\in\mathbb{R}^K :\sum_{k=1}^K v_k=1, v_k\geq 0, k=1, \dots, K\right\}.\]
We consider deep ReLU networks with softmax output to model conditional class probabilities. 
Let $\sigma(x)=\max\{x,0\}$ denote the ReLU activation. A network with depth $L$ and widths 
$m=(m_0,\ldots,m_{L+1})$ is a function of the form
\[
\mathbf{x}\mapsto \mathbf{f(x)}=\Phi W_L \sigma_{v_L} \cdots W_1 \sigma_{v_1} W_0 \mathbf{x},
\]
where $(W_j,v_j)$ are weight matrices and shift vectors, and $\Phi$ is the output activation. 
To ensure probabilistic outputs, the softmax function $\Phi:\mathbb{R}^K \to \mathcal{S}_K$ 
is used in the final layer. We denote by $\mathcal{F}_\Phi(L,m,s)$ the class of such networks with parameters bounded by $1$ w.r.t. $\Vert \cdot \Vert_{\infty}$ and sparsity level $s$.  These networks then produce functions $\widehat{\mathbf{p}}:\,[0,1]^\mathbf{d} \to \mathcal{S}_K$  for estimating 
conditional class probabilities.

\begin{assumption}\label{ass:assumption-bos-true-proba}
The true conditional class probability vector $\mathbf{p}_0$ satisfies
    \begin{align}\label{eq:amoothness-bos}
        \mathbf{p}_0\!\in\!\calG(\beta, Q)\!=\!\Big\{\mathbf{p}=(p_1, ..., p_K)^{\top} \!: [0,1]^\mathbf{d}\rightarrow\mathcal{S}^{K} \,\Big| \, p_k\in C^{\beta}([0,1]^\mathbf{d},Q), \,k=1,\dots, K
    \Big\},
    \end{align}
   and, moreover, $\mathbf{p}_0$ satisfies the $\alpha$-SVB condition, meaning that there exists a constant $C>0$ such that
$\PP_X({p}_k(X)\leq t)\leq Ct^{\alpha}$ for all $t\in(0,1]$ and for all $k\in\{1, \dots, K\}$. 
\end{assumption}
For $\alpha\in[0,1]$, the index from the SVB condition above, and $\beta$ the smoothness index in \eqref{eq:amoothness-bos}, we define
\begin{equation}\label{eq:def-psi-bos}
  \psi_{N,\mathbf{d}}=K^{\frac{(1+\alpha)\beta+(3+\alpha)\mathbf{d}}{(1+\alpha)\beta+\mathbf{d}}} N^{-\frac{(1+\alpha)\beta}{(1+\alpha)\beta+\mathbf{d}}}.   
\end{equation}
\begin{assumption}\label{ass:assumption-bos-nn}
The neural network class $\mathcal{F}_\Phi(L,m,s)$ satisfies
    \begin{align}
        A(\mathbf{d},\beta)\log_2(N)\leq L\lesssim N \psi_{N,\mathbf{d}}, \quad \min_{i=1, ..., L} m_i \gtrsim N\psi_{N,\mathbf{d}}, \quad s\asymp N \psi_{N,\mathbf{d}} \log(N),
    \end{align}  
    for suitable constant $A(\mathbf{d}, \beta)$. 
\end{assumption}

Moreover, we define the optimization error $\Delta_N(\widehat{\mathbf{p}}, \mathbf{p}_0)$ as the gap between the empirical risk achieved by $\widehat{\mathbf{p}}$ and the minimal empirical risk over the class $\mathcal{F}_\Phi(L,m,s)$:
\[
\Delta_N(\widehat{\mathbf{p}}, \mathbf{p}_0)
\coloneqq
\mathbb{E}_{\mathcal{D}_N}\left[
-\frac{1}{N}\sum_{i=1}^N \mathbf{Y}_i^{\top} \log \widehat{\mathbf{p}}(\mathbf{X}_i)
\;-\;
\min_{\mathbf{p} \in \mathcal{F}_\Phi(L,m,s)}
\left(
-\frac{1}{N}\sum_{i=1}^N \mathbf{Y}_i^{\top} \log \mathbf{p}(\mathbf{X}_i)
\right)
\right].
\]

%These networks produce functions $p:\,[0,1]^d \to \mathcal{S}_K$ suitable for estimating 
%conditional class probabilities.

\subsection{Excess risk of a FNN classifier with trajectory input}\label{sec:bos-prop}

We now consider the trajectory-based classifier, where the entire discretely observed path is used as input. More precisely, each sample is written as  
\[
\mathbf{X}^{[k]}_i =
\big(
\bar X^{[k],(i)}_{t_0},
\dots,
\bar X^{[k],(i)}_{t_M}
\big),
\]
and this vector is directly fed into a feedforward neural network in $\mathcal{F}_\Phi(L,m,s)$. Note that, in this setting, each input has dimension \(\mathbf{d}=(M+1)d\). We now establish an upper bound on the excess risk by applying Theorem 3.3 of \citet{bos2022}.

Let \(\mathbf{p}_0\) denote the true conditional class probability vector defined in
\eqref{eq:true-cond-bos2022} and \eqref{eq:true-cond-vector-bos2022}, and let
\(\widehat{\mathbf{p}}=(\widehat p_1, ..., \widehat p_K)\) be the estimator of
\(\mathbf{p}_0\), constructed by a neural network in the class
\(\mathcal{F}_\Phi(L,m,s)\), satisfying Assumption~\ref{ass:assumption-bos-nn}.  We define the corresponding predicted-label functions by
\[
\mathbf{f}_{\mathbf{p}_0(\mathbf{X})}= \argmax_k p_k^0(\mathbf{X}) \quad \text{ and }\quad \mathbf{f}_{\widehat{\mathbf{p}}(\mathbf{X})}= \argmax_k \widehat{p}_k(\mathbf{X}). 
\]
Note that both \(\mathbf{f}_{\mathbf{p}_0(\mathbf{X})}\) and
\(\mathbf{f}_{\widehat{\mathbf{p}}(\mathbf{X})}\) take values in
\(\{1,\ldots,K\}\). The excess misclassification risk is then given by
$
  \PP (\mathbf{f}_{\widehat{\mathbf{p}}(\mathbf{X})}\neq Y)-   \PP (\mathbf{f}_{\mathbf{p}_0(\mathbf{X})}\neq Y). 
$
\begin{proposition}\label{prop:excess-risk-bos2022}
Assume that Assumptions~\ref{ass:assumption-bos-true-proba} and~\ref{ass:assumption-bos-nn} hold. Fix $M\in\mathbb{N}^*$. 
If $N$ is sufficiently large, then there exist constants $C'$ and $C''$, depending only on $\alpha,C,\beta,d$, such that whenever
$\Delta_N(\widehat{\mathbf{p}},\mathbf{p}_0)
\leq C'' B\psi_{N,(M+1)d} L\log^2(N),
$
we have
\begin{equation}
   \PP (\mathbf{f}_{\widehat{\mathbf{p}}(\mathbf{X})}\neq Y)-   \PP (\mathbf{f}_{\mathbf{p}_0(\mathbf{X})}\neq Y)\lesssim \sqrt{\psi_{N,(M+1)d}\log^3N},
\end{equation}
where $\psi_{N, \mathbf{d}}$ is defined in \eqref{eq:def-psi-bos}.
\end{proposition}

\begin{proof}[Proof of Proposition~\ref{prop:excess-risk-bos2022}]
We first establish a link between the Kullback--Leibler (KL) risk and the excess risk. Using conditional expectation,
\begin{align}
     \PP (\mathbf{f}_{\widehat{\mathbf{p}}(\mathbf{X})}\neq Y)=\EE[\mathbbm{1}_{\{\mathbf{f}_{\widehat{\mathbf{p}}(\mathbf{X})}\neq Y\}}]=\EE\big[\EE[\mathbbm{1}_{\{\mathbf{f}_{\widehat{\mathbf{p}}(\mathbf{X})}\neq Y\}}|\mathbf{X}]\big]=\EE\big[\PP(\mathbf{f}_{\widehat{\mathbf{p}}(\mathbf{X})}\neq Y|\mathbf{X})\big],\nonumber
\end{align}
and for any $\mathbf{x}\in\R^{(M+1)d}$, we have 
\begin{align}
    \PP(\mathbf{f}_{\widehat{\mathbf{p}}(\mathbf{X})}\neq Y|\mathbf{X}=\mathbf{x})=1-p^0_{\mathbf{f}_{\widehat{\mathbf{p}}(\mathbf{x})}}(\mathbf{x}). \nonumber%\quad \PP(\mathbf{f}_{{\mathbf{p}_0}(\mathbf{X})}\neq Y|X=x)=1-p^0_{\mathbf{f}_{{\mathbf{p}_0}(\mathbf{x})}}(x), 
\end{align}
The same argument applied to \(\mathbf{f}_{\mathbf{p}_0(\mathbf{X})}\) yields
\begin{align}
\PP &(\mathbf{f}_{\widehat{\mathbf{p}}(\mathbf{X})}\neq Y)-\PP (\mathbf{f}_{\mathbf{p}_0(\mathbf{X})}\neq Y)=\EE\big[\PP(\mathbf{f}_{\widehat{\mathbf{p}}(\mathbf{X})}\neq Y|X)\big]-\EE\big[\PP(\mathbf{f}_{{\mathbf{p}_0}(\mathbf{X})}\neq Y|X)\big]\nonumber\\
&= \EE [p^0_{\mathbf{f}_{{\mathbf{p}_0}(\mathbf{X})}}(\mathbf{X})-p^0_{\mathbf{f}_{\widehat{\mathbf{p}}(\mathbf{X})}}(\mathbf{X})]\nonumber\\
&\leq \EE [p^0_{\mathbf{f}_{{\mathbf{p}_0}(\mathbf{X})}}(\mathbf{X})- \widehat{p}_{\mathbf{f}_{{\mathbf{p}_0}(\mathbf{X})}}(\mathbf{X}) +\widehat{p}_{\mathbf{f}_{{\mathbf{p}_0}(\mathbf{X})}}(\mathbf{X}) - \widehat{p}_{\mathbf{f}_{\widehat{{\mathbf{p}}}(\mathbf{X})}}(\mathbf{X}) + \widehat{p}_{\mathbf{f}_{\widehat{{\mathbf{p}}}(\mathbf{X})}}(\mathbf{X}) -p^0_{\mathbf{f}_{\widehat{\mathbf{p}}(\mathbf{X})}}(\mathbf{X})]\nonumber\\
&\leq \EE [p^0_{\mathbf{f}_{{\mathbf{p}_0}(\mathbf{X})}}(\mathbf{X})- \widehat{p}_{\mathbf{f}_{{\mathbf{p}_0}(\mathbf{X})}}(\mathbf{X})  + \widehat{p}_{\mathbf{f}_{\widehat{{\mathbf{p}}}(\mathbf{X})}}(\mathbf{X}) -p^0_{\mathbf{f}_{\widehat{\mathbf{p}}(\mathbf{X})}}(\mathbf{X})], \nonumber
\end{align}
since the term $\widehat{p}_{\mathbf{f}_{{\mathbf{p}_0}(\mathbf{X})}} (\mathbf{X})- \widehat{p}_{\mathbf{f}_{\widehat{{\mathbf{p}}}(\mathbf{X})}}(\mathbf{X})$ is nonpositive by definition. 
Finally, 
\begin{align}
    \PP &(\mathbf{f}_{\widehat{\mathbf{p}}(\mathbf{X})}\neq Y)-\PP (\mathbf{f}_{\mathbf{p}_0(\mathbf{X})}\neq Y)\leq 2\EE \left[\sum_{k=1}^{K}\big|p^0_{k}(\mathbf{X})- \widehat{p}_{k}(\mathbf{X}) \big|\right]= 2 \EE \left[\big\Vert\mathbf{p}^0(\mathbf{X})- \widehat{\mathbf{p}}(\mathbf{X}) \big\Vert_{1}\right]\nonumber\\
    &\leq 4 \EE \left[\mathrm{TV}\big(\mathbf{p}^0(\mathbf{X}), \widehat{\mathbf{p}}(\mathbf{X}) \big)\right]\leq 2\sqrt{2} \EE \left[\sqrt{\mathrm{KL}\big(\mathbf{p}^0(\mathbf{X}), \widehat{\mathbf{p}}(\mathbf{X})\big )}\right]\leq 2\sqrt{2} \sqrt{\EE \left[\mathrm{KL}\big(\mathbf{p}^0(\mathbf{X}), \widehat{\mathbf{p}}(\mathbf{X}) \big)\right]}\label{eq:bos-pinsker}
\end{align}
by Pinsker's inequality and Jensen's inequality, where $\mathrm{TV}(P,Q)$ denotes the total variation distance between two  probability  measures $P$ and $Q$. 

Following \citet[Theorem~3.3 and the discussion thereafter]{bos2022}, one can restrict the class of estimators to functions taking values in $[e^{-B},1]^K$, with $B\asymp \log N$, which ensures that the truncation parameter $B$ does not affect the convergence rate. Under this construction, the Kullback--Leibler risk satisfies
$\EE \left[\mathrm{KL}\big(\mathbf{p}^0(\mathbf{X}), \widehat{\mathbf{p}}(\mathbf{X}) \big)\right]\leq C\psi_{N,(M+1)d}\log^3 N$. We conclude by plugging this bound into \eqref{eq:bos-pinsker}.
\end{proof}

\subsection{Comparison of our plug-in classifier with trajectory-based classifiers}\label{sec:bos-comparison}

\paragraph{Theoretical comparison.}
Proposition~\ref{prop:excess-risk-bos2022} yields, up to logarithmic factors, the convergence rate
\[
N^{-\frac{1}{2}\frac{(1+\alpha)\beta}{(1+\alpha)\beta+(M+1)d}}
\]
for the excess risk, by the definition of \(\psi_{N,\mathbf d}\) in \eqref{eq:def-psi-bos}. 
This bound suffers from the curse of dimensionality: when the entire trajectory is used as input, the effective input dimension is \(\mathbf d=(M+1)d\). Hence, for large \(M\), the convergence rate becomes very slow. This is consistent with our numerical findings in Figure~\ref{fig:NN-basedVSB-spline-based_and_NN-basedVSdirectNN}.

For other trajectory-based classifiers, such as RNNs, TCNs, and Transformers, 
we are not aware of directly comparable Bos--Schmidt-Hieber-type nonparametric 
convergence-rate bounds. Existing results for 
recurrent architectures, such as \citet{chen2020generalization}, provide 
PAC-style complexity-based bounds, which are not directly comparable to the 
excess-risk rate considered here. We therefore include these methods in the 
numerical comparison, but not in the theoretical comparison.

\paragraph{Numerical comparison.}
From a numerical perspective, there is a clear difference in the effective use of data. Direct classifiers treat each trajectory as a single training sample, and therefore use only \(N\) samples for classification, with each sample lying in dimension \(\mathbf d=(M+1)d\). By contrast, our plug-in approach learns the drift from trajectory increments, which provides roughly \(NM\) effective data points, where \(M\) is the number of observation intervals. This makes the proposed approach more data-efficient.

\newpage

\end{document}